\documentclass{article}

\usepackage{microtype}
\usepackage{graphicx}
\usepackage{subfigure}
\usepackage{booktabs} % for professional tables

\usepackage[hyphens]{url} %Harald: I added this to break the url in references
\usepackage{hyperref} %some super weird pdflatex stuff without the draft option, see https://tex.stackexchange.com/questions/1522/pdfendlink-ended-up-in-different-nesting-level-than-pdfstartlink 
\usepackage[accepted]{icml2020}

\usepackage[utf8]{inputenc} % allow utf-8 input
\usepackage[T1]{fontenc}    % use 8-bit T1 fonts
\usepackage{hyperref}       % hyperlinks
\usepackage{url}            % simple URL typesetting
\usepackage{booktabs}       % professional-quality tables
\usepackage{amsfonts}       % blackboard math symbols
\usepackage{nicefrac}       % compact symbols for 1/2, etc.
\usepackage{microtype}      % microtypography
\usepackage{caption}
\usepackage{mathtools}
\mathtoolsset{showonlyrefs}

\usepackage[mathcal]{eucal}
\usepackage{amsmath}
\usepackage{amsthm}
\usepackage{textcomp}
\usepackage{amssymb}
\usepackage{todonotes}      %[disable]
\usepackage{siunitx}
\usepackage{mathtools}
\usepackage[inline]{enumitem}
\mathtoolsset{showonlyrefs=true}
\usepackage{multirow}

\usepackage{tikz}
\usetikzlibrary{shapes}

\pgfdeclarehorizontalshading{red}{150bp}{rgb(0bp)=(1,0.76,0.77);
rgb(37.5bp)=(1,0.76,0.77);
rgb(62.5bp)=(0.95,0.55,0.56);
rgb(100bp)=(0.95,0.55,0.56)}

\pgfdeclarehorizontalshading{grey}{150bp}{rgb(0bp)=(0.89,0.89,0.89);
rgb(37.5bp)=(0.89,0.89,0.89);
rgb(50bp)=(0.86,0.86,0.86);
rgb(50.25bp)=(0.82,0.82,0.82);
rgb(62.5bp)=(0.97,0.97,0.97);
rgb(100bp)=(1,1,1)}

\pgfdeclarehorizontalshading{blue}{150bp}{rgb(0bp)=(0.76,0.89,1);
rgb(37.5bp)=(0.76,0.89,1);
rgb(62.5bp)=(0.58,0.79,0.9);
rgb(100bp)=(0.58,0.79,0.9)}

\pgfdeclarehorizontalshading{green}{150bp}{rgb(0bp)=(0.71,0.94,0.73);
rgb(37.5bp)=(0.71,0.94,0.73);
rgb(62.5bp)=(0.49,0.79,0.48);
rgb(100bp)=(0.49,0.79,0.48)}

\tikzset{outer sep=0.1pt, inner sep=0,
	node/.style={rectangle, draw=black, minimum width=42, minimum height=20, outer sep=1.0pt, inner sep=0},
	every picture/.style={line width=0.75pt}}

\icmltitlerunning{Gaussian Processes with Signature Covariances}

\newcommand{\bbE}{\mathbb{E}}
\newcommand{\E}{\mathbb{E}}
\newcommand{\bbN}{\mathbb{N}}

\newcommand{\bbR}{\mathbb{R}}
\newcommand{\R}{\mathbb{R}}

\newcommand{\GP}{\mathcal{GP}}
\newcommand{\cH}{\mathcal{H}}

\newcommand{\cX}{\mathcal{X}}

\newcommand{\kernel}{\operatorname{k}}

\newcommand{\lattice}[2]{\cX_{lattice}^{{#1},{#2}}}

\newcommand{\KL}[2]{D_{KL}\left[#1 \;\|\; #2\right]}

\newcommand{\p}{\prime}

\newcommand{\bx}{\mathbf{x}}
\newcommand{\bi}{\mathbf{i}}
\newcommand{\bj}{\mathbf{j}}

\newcommand{\by}{\mathbf{y}}
\newcommand{\bz}{\mathbf{z}}
\newcommand{\bX}{\mathbf{X}}

\newcommand{\bZ}{\mathbf{Z}}

\newcommand{\TA}[1]{\prod_{m \ge 0} {#1}^{\otimes m}}

\newtheorem{theorem}{Theorem}
\newtheorem{proposition}{Proposition}

\newtheorem{definition}{Definition}

\begin{document}
% \begin{bibunit}
	\twocolumn[
	\icmltitle{Bayesian Learning from Sequential Data using \\ Gaussian Processes with Signature Covariances}
	
	
	
	\icmlsetsymbol{equal}{*}
	
	\begin{icmlauthorlist}
		\icmlauthor{Csaba Toth}{to}
		\icmlauthor{Harald Oberhauser}{to}
	\end{icmlauthorlist}
	
	\icmlaffiliation{to}{Mathematical Institute, University of Oxford, Oxford, United Kingdom}
	
	\icmlcorrespondingauthor{Csaba Toth}{csaba.toth@maths.ox.ac.uk}
	\icmlcorrespondingauthor{Harald Oberhauser}{harald.oberhauser@maths.ox.ac.uk}
	
	\icmlkeywords{time series, Gaussian Processes, Signatures, Bayesian Machine Learning, Deep Learning, ICML}
	
	\vskip 0.3in
	]
	
	
	
	\printAffiliationsAndNotice{}  % leave blank if no need to mention equal contribution
	
	\begin{abstract}       
		We develop a Bayesian approach to learning from sequential data by using Gaussian processes (GPs) with so-called signature kernels as covariance functions.
    This allows to make sequences of different length comparable and to rely on strong theoretical results from stochastic analysis.
    Signatures capture sequential structure with tensors that can scale unfavourably in sequence length and state space dimension.
    To deal with this, we introduce a sparse variational approach with inducing tensors. 
    We then combine the resulting GP with LSTMs and GRUs to build larger models that leverage the strengths of each of these approaches and benchmark the resulting GPs on multivariate time series (TS) classification datasets.
    	\end{abstract}
	
	\section{Introduction} \label{section:introduction}
	The evolution of some state variable, parameter or object gives naturally gives rise to \textit{sequential data}, which is defined by having a notion of order on the incoming information.
  The ordering relation, or \textit{index set} does not have to represent physical time, but for simplicity we will call it as such.
  For example, besides time series, sources of sequential data are text \cite{Pennington2014Glove}, DNA \cite{Heather2016DNA}, or even topological data analysis \cite{chevyrev_nanda_oberhauser_2018}.
  This ubiquity of sequential data has received special attention by the machine learning community in recent years.
  This paper is motivated by the following three approaches:
	% Some challenges in learning from sequential data in the context of machine learning are the following:

	% \begin{description}
	% 	\item[high-dimensionality] The evolution of an underlying system is often only described by a large number of variables interacting over very long time-periods.
		
	% 	\item[asynchronous sampling] The time-points at which distinct variables are observed can vary from instance to instance, and even the sampling frequency can change between different periods in time.  
		
	% 	\item[parametrization (ir-)relevance] It is often only the traced out trajectory that matters, not the speed at which it is traversed, e.g.~when travelling only the route taken between two points, and not the time of arrival.
		
	% 	\item[multi-modality]  Different modes of activity can occur in space and time: in space, variables might evolve at different speed and scales; in time, different time periods can be governed by distinct regimes.
	% \end{description}
	
 \paragraph{Deep learning approaches.} 
  Deep learning approaches, such as the celebrated LSTM network \cite{hochreiter1997long}, other forms of RNNs \cite{Cho2014} and convolutional networks have successfully been applied to a variety of tasks involving sequential data \cite{Sutskever2014Seq2Seq, Oord2016Wavenet}. Deep learning models can approximate any continuous function, but the cost is a large number of parameters, high variance and poor interpretability.
  This leaves the door open for alternative approaches not only as competitors, but as complementary building blocks in a larger model.
	
  \paragraph{Bayesian approaches.}
  Often not only point predictions, but estimates of the associated uncertainties are required \cite{Ghahramani2013Bayesian}.
  GPs \cite{Rasmussen2006Gaussian} provide flexible priors over functions of the data in nonparametric Bayesian models.
  % scale their complexity with the dataset size \cite{Rasmussen2001Occam}.
  In the context of sequential data, two prominent ways to use GPs are: \begin{enumerate*}[label=(\arabic*)] \item using as covariance functions kernels specifically designed for sequences~\cite{lodhi2002text, Cuturi2011GA, Cuturi2011AR, AlShedivat2017Recurrent}, \item  modelling the evolution in a latent space, that emits the observations, as a discrete dynamical system with a GP prior on the transition function, a model called the Gaussian Process State Space Model (GPSSM) \cite{Frigola2013MCMC, Frigola2014Variational, Mattos2016Recurrent, Eleftheriadis2017Identification, Doerr2018Proba, Ialongo2019Overcoming}. \end{enumerate*}
  These two approaches are not mutually exclusive; if one models the latent system as a higher order Markov process, then sequence kernels can incorporate the effect of past states.
	
  \paragraph{Signature approaches.}
  The \textit{signature feature map} is a well-developed tool from stochastic analysis that represents a sequence as an element in a linear space of tensors, \cite{chen-58, Lyons2007Differential}.
  While not a mainstream machine learning approach, it is gaining attention since it can represent \emph{non-linear} functions of sequences as \emph{linear} functions of signature features, and can be made invariant to parametrization similar to dynamic time warping (DTW).
  For example,~\cite{Kidger2019DeepSig} use them as layer in a deep learning architecture;~\cite{KiralyOberhauser2019KSig} introduce kernels for sequences by taking inner products of signature features;
  \cite{ChevyrevOberhauser18} use them for maximum mean discrepancies between laws of stochastic processes. In particular, if $\kappa: \R^d \times \R^d \rightarrow \R$ is a kernel for vector-valued data, then \cite{KiralyOberhauser2019KSig} uses signatures to derive the following kernel
  \begin{align*}
      \kernel(\bx, \by) = \sum_{m=0}^M \sigma_m^2 \sum_{\bi_m, \bj_m} c(\bi_m) c(\bj_m) \prod_{l=1}^m \Delta_{i_l, j_l} \kappa(\bx_{i_l}, \by_{j_l}), 
  \end{align*}
  for two sequences $\bx = (\bx_i)_{i=1}^{l_\bx}$ and $\by = (\by_j)_{j=1}^{l_\by}$, with the double difference operator defined as $\Delta_{i, j} \kappa(\bx_i, \by_j) := \kappa(\bx_{i+1}, \by_{j+1}) - \kappa(\bx_{i}, \by_{j+1}) - \kappa(\bx_{i+1}, \by_{j}) + \kappa(\bx_{i}, \by_{j})$, and the sums are taken over multi-indices $\bi_m=(i_1,\ldots,i_m)$ with $1 \leq i_1 \le \cdots \le i_m < l_\bx$ (and analogous for $\bj_m$) and some explicitly computable coefficients $c(\bi) \in [0, 1]$.

	\paragraph{Our contribution.}
  In principle, one can just use the signature kernel and algorithms from~\cite{KiralyOberhauser2019KSig} as covariance to define a GP for sequential data. 
  However, the computational complexity becomes quickly prohibitive and the low-rank approximation too crude, which ultimately does not lead to competitive results on many TS benchmarks.
  We therefore develop a different approach to signature covariances that builds on two recent advances in GP inference, namely variational inference \cite{Titsias2009Variational, Hensman2015Scalable, matthews2016sparse} and inter-domain inducing points~\cite{lazaro2009inter} to alleviate the computational burden.
  In particular, we show that one can use sparse tensors as inter-domain inducing points by optimizing a variational bound.
  Moreover, we use this GP as a building block in combination with RNNs to build models that combine the strenghts of these different tools. 
  This results in scalable inference algorithms and we use this to benchmark on standard TS datasets \begin{enumerate*}[label=(\roman*)] \item against popular non-Bayesian time series classifiers purely in terms of accuracy, \item against alternative Bayesian models by comparing the calibration of uncertainties for predictions \end{enumerate*}.
    Code and benchmarks are publically available at \url{http://github.com/tgcsaba/GPSig}.

\section{Background and Notation}\label{sec:background}
Given data $(\bX, Y)$ consisting of $n_\bX$ inputs $\bX = (\bx_1,\ldots,\bx_{n_\bX}) \subset \cX$ with labels $Y=(y_1,\ldots,y_{n_\bX}) \in \R$, a common Bayesian approach is to put a prior on a set of functions $\{f | f: \cX \rightarrow \R\}$, update this prior by conditioning on $(\bX, Y)$, and then use the resulting posterior to make inference about the label $y_\star$ of an unseen point $\bx_\star$. 
When this is done with Gaussian priors, the central object is a GP $f=(f_\bx)_{\bx \in \cX}$ which is specified by mean and covariance function.
Throughout we interchangibly use the notation $f_\bx$ and $f(\bx)$.
Below we recall how covariances can be constructed from feature maps and discuss the case when $\cX$ is a space of sequences of arbitrary length.

\paragraph{The feature space view.}
Given a map $\varphi:\cX \hookrightarrow V$ that injects $\cX$ into a linear space $V$, a natural way to put a prior on a function class $\cX \rightarrow \R$ is to consider \emph{linear} functions of $\varphi$ as model, that is $f(\bx):= \langle \ell, \varphi(\bx) \rangle$ to model $f(\bx_i)\approx y_i$ for some ``weights`` $\ell \in V$. Uncertainty about $f$ is then specified by uncertainty about $\ell$ (and the hyperparameters of $\varphi$). 
We refer to $\varphi$ as a \emph{feature map} and to $V$ as \emph{feature space}.
% The advantage of a nonparametric approach is that it allows one to avoid explicitly computing features and weights, and to perform inference directly on the function $f$.
Throughout we assume that $f=(f_\bx)_{\bx \in \cX}$ is a centered GP and predictions about unseen points can then be made by Gaussian conditioning.    
If the task is classification where the labels $Y$ are discrete, such an approach can be still applied by using a GP $f=(f_\bx)_{\bx \in \cX}$ as \emph{nuiscance function} to put a \emph{prior on the class membership probability} by specifying $ p(y=1|\bx)=\sigma(f(\bx))$ where $\sigma$ is for example a sigmoid.

\paragraph{Polynomial features.} 
The classical example is $\cX=\R^d$ and
\begin{align}\label{eq:moments}
 \varphi(\bx):=(1,\bx, \bx^{\otimes 2}, \bx^{\otimes 3}, \ldots, \bx^{\otimes M})
\end{align}
where $\bx^{\otimes m} \in (\R^d)^{\otimes m}$ is a tensor. 
We recall background on tensors in Appendix \ref{app:tensors}. 
If we set $f(\bx)= \langle  \ell, \varphi(\bx)\rangle$ and put a centered Gaussian prior on $\ell = (\ell_1,\ldots, \ell_M)$, then by linearity of the tensor product and expectation it follows that
\begin{align}\label{eq:polynomial cov.}
\textstyle{\E[f_\bx f_\by] =1+ \sum_{m=1}^M \langle \Sigma_m^{2}, \bx^{\otimes m} \otimes \by^{\otimes m} \rangle}
\end{align}
where $\Sigma_m^2:= \E[ \ell_m \otimes \ell_m ] \in (\R^d)^{\otimes (2m)}$.
Taking $\Sigma_m^2$ to be an isotropic diagonal matrix $\sigma_m^2 \cdot I^{\otimes m}$ recovers the polynomial kernel,
\begin{align}
\textstyle{\E[f_\bx f_\by]} = \sum_{m=0}^M \sigma_m^2 \langle \bx, \by \rangle^m 
\end{align}
where we use the convention $\langle \bx,\by \rangle^0= \sigma_0^0=1$.
Many variations exist, for example other classes of polynomials, such as Hermite polynomials (the eigenfunctions of the classic RBF kernel), can increase the effectiveness, since they allow to make the associated feature expansion infinite dimensional.
However, what makes any such class of polynomials a sensible choice for $\varphi$ is that by the Stone--Weierstrass theorem, any continuous compactly supported function $\cX \rightarrow \R$ can be arbitrary well approximated as linear functions of $\varphi(\bx)$.
This approximation property often runs under the name \emph{universality} \cite{Micchelli2006Universal, sriperumbudur2011universality}.

\paragraph{Sequences as paths.}
We study the case when one observation $\bx$ is a sequence of $l_\bx$ tuples, $\bx=(x_i,t_i)_{i=1,\ldots,l_\bx}$, and each tuple $(x_i,t_i)$ specifies that at ``time'' $t_i$ a vector $x_i \in \R^d$ was measured.
We denote with $\cX_{seq}$ the set of all such sequences.
% \begin{align}
% \cX_{seq}=\{&\bx=(t_i,x_{i})_{i=1,\ldots,l_\bx}: \,(t_i,x_{i}) \in \bbN \times \R^d, \right. \\ \left. &t_1< \cdots < t_{l_\bx}, l_\bx \in \bbN\}.
% \end{align}
and emphasize that the length $l_\bx \ge 1$ is not fixed, which is a common case in real-world data.
%o follow the feature-space view we need a feature map that makes sequences of different length comparable.
We now introduce an even larger set than $\cX_{seq}$: 
sequential data evolves in discrete time but often arises by sampling a quantity that evolves in continuous time.
Thus above the set $\cX_{seq}$ of sequences lurks the larger set of finite horizon paths 
\begin{align*}
    &\cX_{paths} = \\ &\quad\{\bx \in C([0,t_\bx], \R^d): t_\bx \in \R^+, \bx(0) =0, \| \bx \|_{bv}<\infty\},
\end{align*}
which are simply continuous $\R^d$-valued functions on some bounded time-interval. 
Here $\|\bx \|_{bv}:=\sup_{0\le t_1<\cdots<t_n\le t_\bx} \sum_{i=1}^{n} |
\bx(t_{i+1})-\bx(t_i)|$ denotes the usual bounded variation norm\footnote{Our approach generalizes to much rougher paths, such as Brownian trajectories and we give details in Appendix~\ref{app:signatures}.}. 
The set $\cX_{seq}$ naturally embeds into $\cX_{paths}$ by mapping a sequence $\bx \in \cX_{seq}$ to the path that is given by linear interpolation between the points $(0,0),(t_1,x_1),\ldots,(t_{l_\bx},x_{l_\bx})$; formally $\bx \in \cX_{seq}$ maps to the element of $\cX_{paths}$ defined as
\begin{align}\label{eq: pcw linear}
\textstyle{ t \mapsto (t_{i+1} - t_{i})^{-1}\left(x_{{i}}(t_{i+1} - t) + x_{{i+1}}(t - t_{i})\right)} 
\end{align} 
for $t \in [t_i,t_{i+1})$.
Henceforth, we implicitly use this embedding, i.e.~with slight abuse of notation, given $\bx \in \cX_{seq}$ we also write $\bx \in \cX_{paths}$. 
Key to our approach, and what makes it different to classic approaches, is to define a GP indexed by the larger set $\cX_{paths}$ rather than just $\cX_{seq}$.
At first, this looks wasteful since $\cX_{paths}$ is much bigger than $\cX_{seq}$ and in practice one only has access to discrete time data (already for storage reasons).
But on a theoretical side, going from discrete time to continuous time has two big advantages:
\begin{enumerate*}[label=(\roman*)]
  \item by construction, such a GP is consistent in the high-frequency limit (that is when we sample an object evolving in continuous time at higher and higher frequencies),
  \item we can make use of well-developed theory from stochastic analysis; in particular we use so-called signature features for paths.
\end{enumerate*}
%In Section~\ref{sec:our GP} we show that signature features give rise to such a GP $(f_\bx)_{\bx \in \cX_{paths}}$.
\section{From signature features to covariances}\label{sec:our GP}
The signature feature map $\Phi$ can be seen as a generalization of the polynomial feature map $\varphi$ as defined in~\eqref{eq:moments} from the domain $\cX=\R^d$ of vectors to the domain of paths $\cX_{paths}$.
It is defined as
\begin{align}\label{eq: signature}
 \textstyle{\Phi(\bx) = (1,\int_0^{t_\bx} d\bx, \int_0^{t_\bx} d\bx^{\otimes 2}, \ldots, \int_{0}^{t_\bx} d\bx^{\otimes m})} 
\end{align}
where $\textstyle{\int_0^{t} d\bx^{\otimes (m+1)}:= \int_0^{t_\bx} \int_0^{s}d\bx^{\otimes m} \otimes d\bx(s) \in (\R^d)^{\otimes m}}$ and $\textstyle{\int_0^{t} d\bx^{\otimes 1}:= \bx(t)}$. 
Signatures are classic objects in stochastic analysis, but probably unfamiliar to researchers in ML and we provide background in Appendix~\ref{app:signatures}. Additionally, we recommend \cite{CK16} for a hands-on introduction to signature features that provides a good complement to our presentation, motivating signatures as a generalization of polynomial features to sequences.

For what follows, only three facts will be used about $\bx \mapsto \Phi(\bx)$:
\begin{enumerate*}[label=(\arabic*)]
\item
  it maps paths of different length to the same space, thus makes paths of different length comparable,  
\item functions of paths can be arbitrary well approximated by \emph{linear} functions of $\Phi(\bx)$,
\item it distinguishes paths that follow different trajectories, but not paths that only differ by parametrization).
\end{enumerate*}
These points explain why signature features $\Phi(\bx)$ are a natural generalization of polynomial features $\varphi(\bx)$: not only do they use the same feature space (sequences of tensors), they also have the same attractive properties such as being able to approximate continuous functions.
Below we discuss how to make $\bx \mapsto \Phi(\bx)$ distinguish paths with different time-parametrizations.

\begin{table}[t]
    \begin{center}
    \begin{small}
        \begin{tabular}{lrr}
        \toprule
        & Vectors & Paths \\
        \midrule
        Domain  & $ \bx \in \R^d$ & $ \bx=(\bx_t)_{t \in [0,t_\bx]} \in \cX_{paths}$\\
        Features & $\varphi(\bx)=(\bx^{\otimes m})_{m}$ & $\Phi(\bx)=(\int d\bx^{\otimes m})_{m}$\\
        Feature space & $\prod_{m}(\R^d)^{\otimes m}$ & $\prod_m(\R^d)^{\otimes m}$\\
        Functions & $f:\R^d \rightarrow \R$ & $f:\cX_{paths} \rightarrow \R$\\ 
        Covariance & $\sum_m \sigma_m^2\langle \bx,\by \rangle^m$& $\sum_m\sigma^2_m \int\int\langle d\bx_s,d\by_t\rangle$\\
        \bottomrule
        \end{tabular}
    \end{small}
    \end{center}
    \label{table:vectors_vs_paths}
    \caption{Comparison of polynomial and signature features}
\end{table}

\paragraph{Parametrization (in)variance.}
A classic empirical finding that led to DTW is that functions of sequences are to a certain degree invariant to the time parametrization: for example, different speakers pronounce words at different speeds. 
However, sometimes the parametrization matters, e.g.~for financial data.
Thus we do not only care about the set of functions 
\begin{align}
	\label{eq:tree-like invariant F}
\textstyle{\{ f:\cX_{paths} \rightarrow \R\,|\, \text{cont.~and compactly supported}\} }
\end{align}
but also about the subset of it that consist of parametrization invariant functions. 
To make this precise, we call $\bx=(x_{t})$ a \emph{reparametrization} of $\by=(y_t)$ if there exists a a smooth increasing function $\rho: [0,t_\bx] \rightarrow [0,t_\by]$ (the ``time change'') such that $x_t=y_{\rho(t)}$ for all $t$.
We call an element $f$ of~\eqref{eq:tree-like invariant F} \emph{parametrization invariant} if $f(\bx)=f(\by)$ when $\bx$ and $\by$ are reparametrizations. 
Often the function we want to learn is invariant to a bit of reparametrization but not extreme reparametrization, so we need a more nuanced way to quantify parametrization (in)variance.
Hence, what we really want is a hyperparameter $\tau \ge 0$ that signifies the degree of parametrization invariance: for $\tau=0$ all the mass of our prior should concentrate on the ``extreme case'' that is the subset of \eqref{eq:tree-like invariant F} consisting of functions that are parametrization invariant; and as $\tau$ gets increased the probability mass should spread out to functions that are sensitive to parametrization. 
This would allow to infer the degree of parametrization invariance by automatic relevance discovery (ARD). 
To accomplish this, we parametrize signature features with a parameter $\tau \ge 0$ as follows
\begin{align}
\textstyle{ \Phi_\tau(\bx):= (1,\int_0^{t_\bx} d\bx_\tau, \int_0^{t_\bx} d\bx_\tau^{\otimes 2},\ldots, \int_0^{t_\bx} d\bx_\tau^{\otimes m}) }
\end{align}
where $\bx_\tau(t):= (\tau \cdot t, \bx(t)) \in \R^{1+d}$.
This simply makes the parametrization part of the trajectory $\{\bx_\tau(t): t \in [0,t_\bx]\} \subset \R^d$ by adding an extra coordinate.
Since signatures distinguish different trajectories (but not the speed at which we run through them), it follows that for $\tau>0$,
\begin{align}
 \Phi_\tau(\bx) = \Phi_\tau(\by) \text{ if and only if } \bx=\by,
\end{align}
and for $\tau = 0$ we have, 
\begin{align}
\textstyle{\Phi_0(\bx) = \Phi_0(\by) \text{ if and only if } \bx \sim \by}
\end{align}
since the extra coordinate is ``switched off''. 
Here, $\sim$ denotes tree-like equivalence, but we invite the reader to read $\bx \sim \by$ as saying that $\bx$ is a reparametrization of $\bx$.
This is strictly speaking not true and we give the precise mathematical statement in Appendix~\ref{app:signatures}, but note that for real-world data, tree-like equivalence is synonymous with reparametrization.

\paragraph{Signature covariances.}
%By the above, we can approximate functions of paths as linear functions of the signature. 
Following the feature space view, we now argue in complete analogy to the case of the classical polynomial feature map $\varphi$ for $\cX=\R^d$, and define a centered GP $f=(f_\bx)_{\bx \in \cX_{paths}}$ by putting a centered Gaussian prior on $\ell=(\ell_1,\ldots,\ell_M)$ and setting
\begin{align}
 f_\bx:= \langle \ell, \Phi_\tau(\bx) \rangle. 
\end{align}
The GP is hence fully specified by its covariance function 
\begin{align}\label{eq:signature covariance}
\textstyle{  k(\bx, \by)= \sum_{m=0}^M\langle \Sigma_m^2, \int d\bx_\tau^{\otimes m} \otimes \int d\by_\tau^{\otimes m} \rangle}
\end{align}
that has $(\tau, M, \Sigma_1^2,\ldots,\Sigma_M^2)$ as hyperparameters.
In particular, choosing an isotropic covariance structure for $\Sigma_m^2$ gives
\begin{align} \label{eq:isotropic_sig_cov}
\textstyle{k(\bx, \by) = \sum_{m=0}^M \int \int \sigma_m^2 \langle d\bx_\tau, d\by_\tau \rangle^m,}
\end{align}
where the integrations are over the simplices $0 < t_1 < \dots < t_m < t_{\bx}$ and $0 < s_1 < \dots < s_m < t_{\by}$. Furthermore, for the special case of paths that arise as linear interpolation of sequences $\bx, \by \in \cX_{seq}$, this covariance reduces to iterated sums of increments
\begin{align} \label{eq:discrete_sig_cov}
\textstyle{	k(\bx, \by) = \sum_{m=0}^M \sum_{\bi_m, \bj_m} \sigma_m^2 c(\bi_m, \bj_m) \prod_{l=1}^m \langle \Delta \bx_{i_l}, \Delta \by_{j_l} \rangle}
\end{align}
for some explicitly known constants $c(\bi_m, \bj_m)$, where the inner sums are over all $m$-tuples $\bi_m = (i_1, \dots i_m)$ and $\bj_m = (j_1, \dots, j_m)$ with $i_1 \leq \dots \leq i_m$ and $j_1 \leq \dots \leq j_m$, and the convention that the empty summation is equal to $1$.
%Note that taking $c(\bi_m, \bj_m) = 1$ if there are no repeating indices in $\bi_m$ and $\bj_m$, and otherwise to $0$ gives a good approximation \cite{KiralyOberhauser2019KSig}.

\paragraph{GP regularity.}
One expects that the GP with covariance \eqref{eq:signature covariance} has nice regularity properties, however, the index set $\cX_{paths}$ is a very large space so some care is needed. 
In Appendix~\ref{app:gpsig}, we compute covering numbers that yield explicit bounds on the modulus of continuity in terms of the path path length $\|\bx\|_{bv}$. 
\begin{theorem} \label{thm:continuity}
  Let $L > 0$ and $\cX_{paths}^L:=\{\bx \in \cX_{paths}: \| \bx \|_{bv} \le L\}$.
  There exists a centered GP $f=(f_\bx)_{\bx \in \cX^L_{paths}}$ with $k(\bx,\by)$ as defined in~\eqref{eq:signature covariance} as covariance function and that has continuous sample paths $\bx \mapsto f_\bx$.
  Further, an explicit bound on its modulus of continuity in terms of $L$ is given in equation~\eqref{eq:modulus}. 
\end{theorem}
We now have a well-defined GP for Bayesian inference for sequences at hand that inherits many of the attractive properties of signature features. 
To turn this into useful models for large TS benchmarks we develop efficient inference algorithms in the next section.

%$$Unfortunately, it does not scale for many TS benchmarks: computing $\int d\bx^{\otimes m}$ directly requires $O(t_\bx d^M)$ steps and $O(d^M)$ memory; the low-rank algorithms from~\cite{KiralyOberhauser2019KSig} reduce the 
		\section{Sparse variational inducing tensors}\label{sec:sparse var tensor}
To reiterate, we are given data $(\bX, Y)$ consisting of $n_\bX$ sequences $\bX = (\bx_1,\ldots,\bx_{n_\bX})$ of maximal length $l_\bX:=\max_{\bx \in \bX}l_\bx \subset \cX_{seq}$ that evolve in $\R^d$ with labels $Y=(y_1,\ldots,y_{n_\bX})$, and the task is to predict labels $y_\star$ of unseen points $\bx_\star$.
%	Our GP $f=(f_\bx)_{\bx \in \cX_{paths}}$ can be used as {nuiscance function} to put a {prior on the class membership probability} by identify sequences as paths and specifying $ p(y=1|\bx)=\sigma(f(\bx))$ where $\sigma$ is for example a sigmoid.
	%$ y | f(\bx) \sim \operatorname{Ber}(\sigma(f(\bx)))$ where $f \sim \GP(m,k)$ is a GP model and $\sigma : \bbR \rightarrow [0,1]$ is a non-linear function.
For sequential data, the sample size $n_\bX$ and associated covariance matrix inversion is not the only compational bottleneck but also the maximal length of sequences $l_\bX$, and the dimension $d$ of the state space matter: $n_\bX$, $l_\bX$ and $d$ can be simultaneously large. 

In this section, we introduce a sparse inference scheme to approximate the posterior of our GP, that locates the inducing points in a space other than the data-domain; an approach that is usually coined the term \textit{inter-domain} sparse variational inference \cite{lazaro2009inter, matthews2016sparse}. This allows for more efficient data-representation and faster inference. 
Key to our approach is that signature features take values in a well-understood subset of the feature space $\prod_{m \ge 0} (\R^d)^{\otimes m}$.
This allows as us to augment the index set with structured tensors, and locate inducing points in this larger index set.

    \paragraph{Variational inference.}
	As is well-known, inference for GPs scales as $O(n_\bX^3)$, see Section 3.3.~in~\cite{Rasmussen2006Gaussian}.
	This first led to \emph{sparse} models, \cite{quinonero2005unifying}, that select a subset $\bZ=\{\bz_1,\ldots,\bz_{n_\bZ}\}$ of $\bX$ consisting of $n_\bZ\ll n_\bX$ points, and subsequently to \emph{pseudo-inputs}, \cite{snelson2006sparse}, that select points $\bZ$ that are not necessarily in $\bX$.
	This was a big step towards complexity reduction, but pseudo-inputs are prone to overfitting, \cite{MatthewsDPhil}.
	A different idea is to treat $\bZ$ as parameters of a variational approximation \cite{Titsias2009Variational} and not as model parameters; that is the points $\bZ$ are choosen simultaneously with the hyperparameters of the GP by maximising a lower bound on the log-marginal likelhood $\log p(Y)$, the so-called \emph{evidence lower bound} (ELBO), given as
	\begin{align} \label{eq:elbo}
		\log p(Y) \geq \E_{q(f_\bX)}[\log p(Y \vert f_\bX)] - \KL{q(f_\bZ)}{p(f_\bZ)},
	\end{align}
	where $f_\bX$ and $f_\bZ$ denotes the GP evaluated at the data-points and the inducing locations. Typically, $q(f_\bZ)$ is given a free-form multivariate Gaussian to be learnt from the data, and then extended to other indices of the GP by \textit{prior conditional matching}, i.e. $q(f_\bX \vert f_\bZ) = p(f_\bX \vert f_\bZ)$.	Initially applied to regression, this was extended to classification~\cite{chai2012variational, Hensman2015Scalable}.
	Among its advantages are that it gives a nonparametric approximation to the true posterior, adding inducing points only improves the approximation, and any optimization method can be used to maximize the ELBO, most importantly, stochastic optimization; see \cite{Hensman2013GaussianPF, bauer2016understanding,bui2016unifying}.

%The prominent approximate inference scheme is variational inference \cite{Blei2017Variational} with the approximate posterior given by a sparse variational Gaussian process (SVGP) %\cite{Titsias2009Variational, Hensman2015Scalable, matthews2016sparse}. The SVGP is characterized by a finite-set of inducing features with a joint free-form multivariate Gaussian distributions, and extended to other indices by prior conditional matching \cite{MatthewsDPhil}. The inducing features need not correspond to point-evaluations of the GP, e.g. \cite{lazaro2009inter, Hensman2016Fourier} use Fourier features, \cite{Wilk2017ConvGP} uses inducing patches.
\paragraph{Inter-domain approaches.}Another idea is to go beyond the original index set and place inducing points $\bZ$ in a different space $\cX'$, that is, given a centered GP $g=(g_\bx )_{\bx \in \cX}$ one augments the original index set $\cX$ by a set $\cX'$ to define a new GP $(g_\bx)_{\bx \in \cX \cup \cX'}$ and then locates the inducing points in this bigger model.
	This was suggested in~\cite{lazaro2009inter} in the context of integral transforms, which was extended in~\cite{Hensman2016Fourier}, and studied in more generality in~\cite{matthews2016sparse}.
	In general, it is not obvious how to find a useful augmentation set $\cX'$ and define the covariance enlarged to $\cX\cup \cX'$.
  \subsection{A sparse variational tensor augmentation.}
	Given any GP with a covariance function $\kernel(\bx,\by):=\langle \Phi(\bx), \Phi(\by) \rangle$ where $\Phi$ is explicitly known\footnote{Mercer's Theorem guarantees the existence of $\Phi$, but not in a sufficiently explicit form.}, we propose that a natural augmentation candidate is the ``feature space'' $\cX':=\operatorname{span}\{\Phi(\bx): \bx \in \cX\}$ itself.
	The covariance function $\kernel$ of $g$ can be simply extended to $\cX\cup \cX'$ by linearity,
  \begin{align}\label{eq:augment}
   \kernel(\bx,\bz):=\kernel(\bz,\bx):=\alpha \kernel(\bx, \bx') + \beta \kernel(\bx, \bx'') 
  \end{align}
   for $\bx \in \cX$, $ \bz=\alpha \Phi(\bx')+\beta \Phi(\bx'') \in \cX'$, $\alpha,\beta \in \bbR$; analogous for $\kernel(\bz,\bz')$ with $\bz,\bz' \in \cX'$. 
For our GP, \[\textstyle{\cX'=\operatorname{span}\{\Phi_{\tau}(\bx):\bx \in \cX_{paths}\} \subset \TA{V}}\] where we denote $V:=\R^d$. 
	%Moreover, an explicit algebraic characterisation of $\{\Phi(\bx): \bx \in \Seq{\cX}\}$ as a subset of $\TA{W}$ is known,~\cite{reutenauer-93} .
We can thus extend our signature covariance~\eqref{eq:signature covariance} to $\cX_{seq} \cup \TA{V}$ by~\eqref{eq:augment}. 
This provides a flexible class of inducing point locations $\bZ$ by optimizing over elements of the tensor algebra $ \bZ \subset \TA{V}$.
We coin these inducing point locations as \textit{inducing tensors}.
	
	\paragraph{Consistency of augmentation.}	
	A subtle point about augmenting the index set is that maximizing the ELBO in \eqref{eq:elbo} is not necessarily equivalent anymore to minimizing a rigorously defined KL divergence between the true posterior process and its approximation over the unaugmented index set. In \cite{matthews2016sparse}, a sufficient condition given for this to hold is that the prior GP evaluated at the newly added indices is deterministic conditioned on the original GP.
  In the case of~\eqref{eq:augment}, this is easily seen to be true, since the augmented indices arise as linear combinations of elements in the original index set. 
Therefore, the corresponding GP evaluations arise as linear combinations of evaluations of the original process by the fact that the feature space $\cX^\p$ is a \textit{representation} of the \textit{Hilbert space generated by the process} \cite{berlinet2003reproducing}.
	
	\paragraph{Representation of inducing tensors.}
	We define our sparse inducing tensors as 
	\[ \label{eq:sparse_tens}
	\textstyle{\bz = (z_m)_{m=0,\ldots,M} \in \prod_{m=0}^M V^{\otimes m}},
	\]
  where $z_0 \in \bbR$ and $z_m = v_{m,1} \otimes v_{m,2} \otimes \cdots \otimes v_{m,m}$ for ${m \geq 1}$.
  We remark that this construction does not generally give tensors that can be signatures of paths.
  However, they can be represented as linear combinations of signatures, hence the previous argument about the augmentation carries over. Also, informally, what gives the data-efficiency of inducing tensors is exactly that they are not represented in a basis of signatures, but as sparse tensors.
	%which leads to a compact representation of the data independent of path-space.

By linearity of integration and the inner product, the inducing point covariance $\bbE[f_{\bz}f_{\bz'}]$ equals the inner product 
	\begin{align} \label{eq:inducing_cov_n}
\textstyle{ \sum_{m=0}^M \sigma_m^2 \langle z_m, z_m^\p \rangle = \sum_{m=0}^M \sigma_m^2 \prod_{k=1}^m\langle v_{m,k}, v_{m, k}^\p \rangle},
	\end{align}
%\langle w_{n,1}, w_{n, 1}' \rangle \langle w_{n,2}, w_{n, 2}' \rangle \cdots \langle w_{n,n}, w_{n,n}' \rangle.
the cross-covariance $\bbE[f_{\bx}f_{\bz}] =\sum_{m=0}^M \sigma_m^2 \langle \Phi_m (\bx), z_m \rangle$,
	\begin{align}\label{eq: cross_cov}
	% &\langle \int_{0 < t_1 < t_2 < \dots < t_n < t_{l_\bx}} d\varphi(x_{t_1}) \otimes \cdots \otimes d\varphi(x_{t_n}), w_{n,1} \otimes  \cdots \otimes w_{n,n} \rangle \\
	%     &= \int_{0 < t_1 < t_2 < \dots < t_n < T} \langle dx_{t_1} \otimes dx_{t_2} \otimes \dots \otimes dx_{t_n}, v_1 \otimes v_2 \otimes \dots \otimes v_n \rangle \\
	\textstyle{\langle \Phi_m (\bx), z_m \rangle= \int \langle dx_{t_1}, v_{m,1} \rangle \cdots \langle dx_{t_m}, v_{m, m} \rangle}
	\end{align}
where the integration is over the simplex ${0 < t_1 < \dots <  t_{l_\bx}}$.
  Finally, note that we just need to evaluate the above for piecewise linear paths since these are the only paths that arise via the embedding~\eqref{eq: pcw linear}, $\cX_{seq} \hookrightarrow \cX_{paths}$.
  For such paths, the above integrals reduce to iterated sums, hence $\langle \Phi_m(\bx), z_m \rangle$ equals
	\begin{align} \label{eq:discr_cross_cov_n}
	%= \sum_{1 \leq i_1 < i_2 < \dots < i_m < l_x}
	%  \prod_{k=1}^n \hspace{5pt} \langle x_{t_{i_k+1}} - x_{t_{i_k}}, z_k \rangle
	\sum_{\bi} c(\bi) \langle x_{{i_1+1}} - x_{i_1}, v_{m,1} \rangle  \cdots \langle x_{{i_m+1}} - x_{{i_m}},  v_{m, m} \rangle,
	\end{align}
  where the sum is taken over all $m$-tuples $\bi=(i_1,\ldots, i_m)$ of the form $1 \leq i_1 \le \cdots \le i_m \le l_\bx$ and $c(\bi)\leq 1$ is given by an explicit calculation.
  Similarly to \eqref{eq:discrete_sig_cov}, replacing $c(i_1,\ldots,i_m)$ with $1$ if there are no repeating indices in $(i_1,\ldots,i_m)$ and otherwise with $0$ gives a good approximation\footnote{It converges to $\langle  \Phi(\bx),\bz \rangle$ when the grid gets finer, $|t_{i+1}-t_i| \downarrow 0$, see \cite{ChevyrevOberhauser18} to~\eqref{eq:discr_cross_cov_n}.} .
Below we use this approximation to~\eqref{eq:discr_cross_cov_n} since it makes the recursive algorithms simpler but note that a simple modification exactly computes~\eqref{eq:discr_cross_cov_n} for a marginal computational overhead.%; a similar approximation was used in~\cite{KiralyOberhauser2019KSig} directly for signatures. 
	\subsection{Algorithms.}
	We need to compute the three covariance matrices:~\begin{enumerate*}[label=(\arabic*)] \item $K_{\bZ\bZ}$ of inducing tensors $\bZ$ and inducing tensors $\bZ$, \label{it:tens2tens} \item $K_{\bZ\bX}$ of inducing tensors $\bZ$ and sequences $\bX$, \label{it:tens2stream} \item $K_{\bX\bX}$ of sequences $\bX$ and sequences $\bX$.\label{it:stream2stream} \end{enumerate*}
	Using the above tensor representations allows to give vectorized algorithms for~\ref{it:tens2tens} and~\ref{it:tens2stream} in Algorithms~\ref{alg:inducing_cov} and \ref{alg:cross_cov}, respectively.
	For~\ref{it:stream2stream} we use a modification of Algorithm 3 from~\cite{KiralyOberhauser2019KSig} which we recall in Appendix \ref{app:stream_covs}.
	We use notation defined in Appendix~\ref{app:alg_notation}, which can be briefly summarized as: $\Sigma$ denotes the slice-wise sum operator, $\boxplus$ the (forward) cumulative sum operator, $+1$ the shift operator, and $\odot$ denotes the element-wise product of arrays.
  Additionally, we set $\Delta x_{i} := x_{i+1} - x_i$ for $i \in \{1, \dots, l_\bx -1 \}$.
  For $v, v^\p \in V$, $d$ denotes the time to compute $\langle v, v^\p \rangle$, $c$ the memory requirement of $v$. 
	
  \begin{proposition}\label{eq:complexity}
	Algorithm~\ref{alg:inducing_cov} computes the covariance matrix $K_{\bZ\bZ}$ of $n_\bZ$ inducing points in $O(M^2 \cdot n_\bZ^2 \cdot d )$ steps.
	Algorithm~\ref{alg:cross_cov} computes the cross-covariance matrix $K_{\bZ\bX}$ in $O(M^2 \cdot n_\bX \cdot n_\bZ \cdot l_\bX \cdot d)$ steps.
  Additionally to storing the inducing tensors $\bZ$, Algorithm~\ref{alg:inducing_cov} requires $O(M^2 \cdot n_\bZ^2)$ memory, Algorithm~\ref{alg:cross_cov} requires $O(M^2 \cdot n_\bX \cdot n_\bZ \cdot l_\bX)$ memory.
  \end{proposition}
  Proposition~\ref{eq:complexity} follows by inspection of the algorithms and we emphasize the following points:
  \begin{enumerate*}[label=(\roman*)]
  \item Both algorithms are linear in the maximal sequence length $l_\bX$.
  \item $M$ is a hyperparameter, and in all our experiments we learnt from the data $M \leq 5$, thus the quadratic complexity in $M$ is negligible.
  \item The memory cost of inducing tensors $\bZ$ is much less than for the data $\bX$, which is stored in $O(n_\bX \cdot l_\bX \cdot d)$ memory, which is important because the inducing tensors are variational parameters, and not amenable to subsampling, while the learning inputs can be subsampled as noted by~\cite{Hensman2013GaussianPF}.
  Especially for GPUs memory cost is decisive and such savings are very important.
  \end{enumerate*}
  
  The computation of $K_{\bX\bX}$ detailed in Appendix \ref{app:stream_covs} has time complexity $O((M + d) \cdot n_\bX^2 \cdot l_\bX^2)$ and memory of $O(d \cdot n_\bX \cdot l_\bX + n_\bX^2 \cdot l_\bX^2)$. However, given a factorizing likelihood, one only requires $K_\bX := [k(\bx, \bx)]_{\bx \in \bX}$, which eliminates the quadratic cost in $n_\bX$.
  It turns out that this is enough to train on GPUs with reasonable minibatch sizes (e.g. $n_\bX = 50$) on several real world datasets. We remark that the low-rank algorithms in \cite{KiralyOberhauser2019KSig} allow to trade off accuracy for linear cost in $l_\bX$, but we found that using the full-rank algorithm performs much better, and the above will allow us to apply it to several datasets with great results. Finally, note that the ELBO \eqref{eq:elbo} requires an additional matrix inversion and multiplication in $O(n_\bZ^2 \cdot n_\bX + n_\bZ^3)$ time, which is not significant in our case.
	
\begin{algorithm}[tb]
		\caption{Computing the inducing covariances $K_{\bZ\bZ}$}
		\label{alg:inducing_cov}
		\begin{algorithmic}[1]
			\STATE {\bfseries Input:} Tensors $\bZ=(\bz_i)_{i=1,\dots,n_\bZ} \subset \prod_{n=0}^m V^{\otimes n}$, \\ scalars $(\sigma^2_0, \sigma^2_1, \dots, \sigma^2_m)$, depth $m \in \bbN$ 
			\STATE Compute $K[i, j, n, k] \gets \langle v^i_{n, k}, v^j_{n, k} \rangle$ for $i, j \in \{1,\dots,n_\bZ\}$, $n \in \{1,\dots,m\}$ and $k \in \{1,\dots,n\}$ \\
			\STATE Initialize $R[i, j] \gets \sigma_0^2$ for $i, j \in \{1,\dots,n_{\bZ}\}$
			\FOR{$n=1$ {\bfseries to} $m$}
			\STATE Assign $A \gets K[:, :, n, 1]$
			\FOR{$k=2$ {\bfseries to} $n$}
			\STATE Iterate $A \gets K[:, :, n, k] \odot A$
			\ENDFOR
			\STATE Update $R \gets R + \sigma_n^2 \cdot A$
			\ENDFOR
			\STATE {\bfseries Output:} Matrix of inducing covariances $K_{\bZ\bZ} \gets R$
		\end{algorithmic}
	\end{algorithm}
	
	\begin{algorithm}[t]
		\caption{Computing the cross-covariances $K_{\bZ\bX}$}
		\label{alg:cross_cov}
		\begin{algorithmic}[1]
			\STATE {\bfseries Input:} Tensors $\bZ=(\bz_i)_{i=1,\dots,n_\bZ} \subset \prod_{n=0}^m V^{\otimes n}$, \\ sequences $\bX=(\bx_i)_{i=1,\dots,n_{\bX}} \subset \cX_{seq}$, \\ scalars $(\sigma^2_0, \sigma^2_1, \dots, \sigma^2_m)$, depth $M \in \bbN$ 
			\STATE Compute $K[i, j, l, m, k] \gets \langle v^i_{m, k}, \Delta x_{j, t_l} \rangle$ for $i \in \{1,\dots,n_\bZ\}$, $j \in \{1,\dots,n_\bX\}$, $l \in \{1, \dots, l_\bX -1 \}$, $m \in \{1,\dots,M\}$ and $k \in \{1, \dots, m\}$ \\
			\STATE Initialize $R[i, j] \gets \sigma_0^2$ for $i \in \{1,\dots,n_{\bZ}\}$, $j \in \{1,\dots,n_{\bX}\}$
			\FOR{$m=1$ {\bfseries to} $M$}
			\STATE Assign $A \gets K[:, :, n, 1]$
			\FOR{$k=2$ {\bfseries to} $m$}
			\STATE Iterate $A \gets K[:, :, :, m, k] \odot A[:, :, \boxplus+1]$
			\ENDFOR
			\STATE Update $R \gets R + \sigma_m^2 \cdot A[:, :, \Sigma]$
			\ENDFOR
			\STATE {\bfseries Output:} Matrix of cross-covariances $K_{\bZ\bX} \gets R$
		\end{algorithmic}
	\end{algorithm}
	
\begin{table*}[t]
	\caption{Average ranks of GPs on datasets \cite{baydogan2015multivarate} with the 1\textsuperscript{st} and 2\textsuperscript{nd} best in bold and italicized for each row}
	\label{table:avg_ranks}
	% \vskip 0.15in
	\begin{center}
		\begin{small}
			\begin{sc}
				\begin{tabular}{lrrrrrr}%{lllll}
					\toprule
					 & GP-Sig-LSTM & GP-Sig-GRU & GP-Sig & GP-LSTM & GP-GRU & GP-KConv1D \\
					\midrule
					Mean rank (nlpp, $n_\bX < 300$) & $\mathit{2.80}$ & $2.90$ & $\mathbf{2.20}$ & $4.70$ & $4.00$ & $4.40$ \\
					Mean rank (acc., $n_\bX < 300$) & $\mathit{3.00}$ & $3.10$ & $\mathbf{2.80}$ & $4.25$ & $4.25$ & $3.60$ \\
					Mean rank (nlpp, $n_\bX \geq 300$) & $\mathbf{2.33}$ & $3.33$ & $\mathit{2.83}$ & $4.83$ & $4.33$ & $3.33$ \\
					Mean rank (acc., $n_\bX \geq 300$) & $\mathbf{2.17}$ & $3.50$ & $\mathit{3.00}$ & $4.17$ & $4.33$ & $3.83$ \\
					Mean rank (nlpp, all) & $\mathit{2.63}$ & $3.06$ & $\mathbf{2.44}$ & $4.75$ & $4.13$ & $4.00$ \\
					Mean rank (acc., all) & $\mathbf{2.69}$ & $3.25$ & $\mathit{2.88}$ & $4.22$ & $4.28$ & $3.69$ \\
					\bottomrule
				\end{tabular}
			\end{sc}
		\end{small}
	\end{center}
\end{table*}

\begin{figure*}
\begin{minipage}{0.44\textwidth}
        \centering
        \begin{tikzpicture}[shorten >=1pt,draw=black!50, node distance=\layersep]

        \node[node, text centered, shading=green] at (0.0,0.5) (x) {$\bx$};
        \node[text centered] at (0.0,1.25) {$(\bbR^d)_{seq}$};
        
        \node[node, text centered, shading=blue] at (1.8,0.5) (kx) {$\kappa_{\hat\bx}$};
        \node[text centered] at (1.8, 1.25) {$V_{seq}$};
        
        \draw[->, color=black] (x) -- (kx);
        
        (4.0,0) rectangle ++(1,1);
        \node[node, text centered, shading=blue] at (3.6,0.5) (sigx) {$\Phi(\kappa_{\hat\bx})$};
        \node[text centered] at (3.6, 1.25) {$\prod_{m=0}^M V^{\otimes m}$};
        
        \draw[->, color=black] (kx) -- (sigx);
        
        \node[node, text centered, shading=red] at (5.4,0.5) (gp) {$f_{\bx}$};
        \node[text centered] at (5.4, 1.25) {$\GP$};
        
        \draw[->, color=black] (sigx) -- (gp);
        \end{tikzpicture}
        \captionof{figure}{The GP-Sig model}
        \label{fig:GPSig}     
\end{minipage}
\begin{minipage}{0.56\textwidth}
        \centering
        \begin{tikzpicture}[shorten >=1pt,draw=black!50, node distance=\layersep]

        \node[node, text centered, shading=green] at (0.0,0.5) (x) {$\bx$};
        \node[text centered] at (0.0, 1.25) {$(\bbR^d)_{seq}$};
        
        \node[node, text centered, shading=blue] at (1.8,0.5) (rnnx) {$\phi_\theta(\hat\bx)$};
        \node[text centered] at (1.8, 1.25) {$(\bbR^h)_{seq}$};
        
        \draw[->, color=black] (x) -- (rnnx);
        
        \node[node, text centered, shading=blue] at (3.6, 0.5) (krnnx) {$\kappa_{\phi_\theta(\hat\bx)}$};
        \node[text centered] at (3.6, 1.25) {$V_{seq}$};
        
        \draw[->, color=black] (rnnx) -- (krnnx);
        
        \node[node, text centered, shading=blue] at (5.4, 0.5) (sigkrnnx) {$\Phi(\kappa_{\phi_\theta(\hat\bx)})$};
        \node[text centered] at (5.4, 1.25) {$\prod_{m=0}^M V^{\otimes m}$};
        
        \draw[->, color=black] (krnnx) -- (sigkrnnx);
        
        \node[node, text centered, shading=red] at (7.2, 0.5) (gp) {$f_\bx$};
        \node[text centered] at (7.2, 1.25) {$\GP$};
        
        \draw[->, color=black] (sigkrnnx) -- (gp);
        \end{tikzpicture}
        \captionof{figure}{The GP-Sig-RNN model}
        \label{fig:GPSigRNN}     
\end{minipage}
\end{figure*}

\paragraph{Variations.}
The following variations produce a more flexible covariance function:
\begin{enumerate*}[label=(\roman*)]
  \item given a nonlinear function $\varphi: \R^d \hookrightarrow V$ into a linear space $V$, lift a sequence $\bx=(t_i, x_i)$ to a path by taking the linear interpolation of $(0,0),(t_1,\varphi(x_1)),\ldots,(t_{l_\bx}, \varphi(x_{l_\bx}))$; with $\varphi$ the identity on $\R^d$ this recovers the original embedding~\eqref{eq: pcw linear},    
  \item 
adding lags is a classic time series pre-processing technique, justified by Takens' theorem, \cite{takens1981detecting}, that guarantees that attractors in a high-dimensional dynamical system can be reconstructed from low-dimensional observations.
\end{enumerate*}
Both points add non-linearities to the feature space which can make the learning more efficient.
If the original sequence evolves in $\R^d$, this preprocessing results in a sequence (and then path) that evolves in a general high-dimensional space $V$.
However, formulas~\eqref{eq:augment},~\eqref{eq:inducing_cov_n}, and~\eqref{eq:discr_cross_cov_n} show that only inner product evaluations on $V$ are used and these can be computationally cheap even if $V$ is high or even infinite dimensional. 
For example, following~\cite{KiralyOberhauser2019KSig} we may take a kernel $\kappa:\R^d \times \R^d \rightarrow \R$ and use $\varphi(x):=\kappa(x,\cdot)$, to build a sequence $\kappa_\bx \in V_{seq}$ in the RKHS of $\kappa$.
The only change in complexity is to replace $d$ in the big $O$-bounds by the cost of the kernel evaluation. 
Such extensions also increase the number of hyperparameters which can have adversarial effects, but in our experiments both extensions led generically to better results. 

\section{Experiments} \label{sec:experiments}
% We used GPFlow~\cite{Matthews2017GPflowAG} and Keras~\cite{Chollet2015Keras} to: \begin{enumerate*}[label=(\roman*)]
% \item
%   build three models, GP-Sig, GP-Sig-LSTM and GP-Sig-GRU and benchmark them on TS classification, 
% \item
%   measure and visualize the usefulness of sparse inducing tensors.
% \end{enumerate*} 

\paragraph{TS classification.}

Using GPFlow~\cite{Matthews2017GPflowAG}, Keras~\cite{Chollet2015Keras}, we implemented three models: GP-Sig, GP-Sig-LSTM, and GP-Sig-GRU.
All three use the signature covariance with the sparse inducing tensors of Section~\ref{sec:sparse var tensor}.
GP-Sig is a plain vanilla variational GP classifier.
Previous applications of neural nets to covariance constructions, in particular~\cite{Wilson2016DeepK, AlShedivat2017Recurrent}, inspired GP-Sig-LSTM and GP-Sig-GRU that include an RNN as a sequence-to-sequence transformation
% $\phi_\theta : \bbR^d_{seq} \rightarrow \bbR^h_{seq}$
with $h $ hidden units; see Figures~\ref{fig:GPSig} and \ref{fig:GPSigRNN} where $\hat \bx$ denotes augmentation with lags and $\kappa_{\hat \bx}$ a static kernel as in above variation paragraph. We benchmarked these GP models on $16$ multivariate TS classification datasets, a collection introduced in~\cite{baydogan2015multivarate} that has become a semi-standard archive in TS classification, e.g. we cite 7 papers in Appendix \ref{app:benchmark} that use these datasets. The same datasets are also used in \cite{Fawaz2019review} to compare several deep learning architectures for TSC.

As Bayesian baselines we used three GP models: \begin{enumerate*}[label=(\roman*)] \item
GP-LSTM and \item GP-GRU consist of an LSTM and a GRU network with an RBF kernel on top, in which case the RNNs are used as a sequence-to-vector transformation from $\bbR^d_{seq}$ to $\bbR^h$; \item GP-KConv1D uses the convolutional kernel introduced in \cite{Wilk2017ConvGP} in $1$-dimension (time) \end{enumerate*}. Throughout we used sparse variational inference: for GP-Sig-LSTM, GP-Sig-GRU, GP-Sig, the inducing tensors detailed in Section~\ref{sec:sparse var tensor} are used; for GP-LSTM and GP-GRU the inducing points are located in the output space of the RNN layer, $\bbR^h$; for GP-KConv1D, the inducing patches of \cite{Wilk2017ConvGP} are used.

We used $n_\bZ = 500$ for all models\footnote{Using $n_\bZ = 500$ is clearly superfluous for small datasets, which is fixed for the sake of consistent settings across datasets.}; further all use a static kernel in one form or another, which we fixed to be the RBF kernel.
The signature kernel was truncated\footnote{For these experiments, the $M=4$ value seemed to give an optimal trade-off between computational complexity and expressiveness of the kernel, see Appendix \ref{app:implementation} for more details.} at $M=4$, and for GP-Sig $p=1$ lags were used; the GP-Sig-RNNs did not use lags, as the sequence of hidden states already incorporate lagged information about past observations. The window size in GP-KConv-1D was set to $w=10$. 
The RNN-architectures were selected independently for all models by grid-search among $6$ variants, that is, the number of hidden units from $[8, 32, 128]$ and with or without dropout. For training, early stopping was used with $n = 500$ epochs patience; a learning rate of $\alpha = 1 \times 10^{-3}$; a minibatch size of $50$; as optimizer Adam \cite{kingma2014adam} and Nadam \cite{Dozat2015IncorporatingNM} were employed. Implementations are detailed in Appendix \ref{app:implementation}, the datasets in Appendix \ref{app:datasets}, the training and grid-search methodology in Appendix \ref{app:training}.
% Combining GPs with kernels constructed from deep learning architectures is not a new idea in general \cite{Wilson2016DeepK}, and in particular, we were inspired by the results of \cite{AlShedivat2017Recurrent}. As baselines, we use GPs with the following kernels: \begin{enumerate*}[label=(\arabic*)] \item an LSTM and \item a GRU network with an RBF kernel on top (GP-LSTM and GP-GRU), in which case the RNNs are used as a sequence-to-vector transformation from $\bbR^d_{seq}$ to $\bbR^h$ by taking the last hidden state. \item A $1$-dimensional variant of the convolutional kernel introduced in \cite{Wilk2017ConvGP}, (GP-KConv1D) \end{enumerate*}.

\paragraph{Discussion of results.}
For GPs, we report accuracies and negative log-predictive probabilities (nlpp), the latter take not only accuracies, but the calibration of probabilities into account as well.
Table \ref{table:avg_ranks} shows the average ranks among the GPs.
The full table of nlpps and accuracies with mean and standard deviation over 5 model trains are reported in Appendix~\ref{app:benchmark} in Table~\ref{table:full_nlpp_results} and Table~\ref{table:full_acc_results}.
As non-Bayesian baselines, we report accuracies of eight frequentist TS classifiers in Table~\ref{table:freq_acc_results}. On Figure \ref{fig:boxplots}, we visualize the box-plot distributions of \begin{enumerate*}[label=(\roman*)] \item negative log-predictive probabilities of GPs, \item accuracies of both GPs and frequentist methods. \end{enumerate*}

The signature models perform consistently the best in terms of average rankings of both nlpp and accuracy among the GPs. Particularly, they achieve stronger mean performance and a smaller variance across datasets. To explain this, inspecting the results in Tables \ref{table:full_nlpp_results}, \ref{table:full_acc_results}, we observe that all other GP baselines perform very poorly on some datasets, while the signature based models perform at least moderately well on \emph{all datasets}. We believe this ties in to the universality property of signatures, see Appendix \ref{app:univ}. The convolutional GP, GP-KConv1D, which also has a very small parameter set, performed rather competitively with the deep kernel baselines, even on larger datasets. Comparison among variants of GP-Sig can be summarized as follows: for smaller datasets ($n_\bX < 300$), GP-Sig outperforms other variants as it has a very small parameter set; for larger datasets $(n_\bX \geq 300)$, GP-Sig-LSTM performs best which conforms with the intuition that RNNs suffer from small sample sizes.
A related observation is that GP-LSTM and GP-GRU perform about on par, while GP-Sig-LSTM does much better than GP-Sig-GRU, which suggests that the signature makes explicit use of the additional gate in the LSTM network.

Compared only in terms of accuracy, GP-Sig competes with frequentist classifiers: it outperforms the usual DTW baseline and competes with state-of-the art classifiers such as MUSE and MLSTMFCN. Purely based on accuracy, these win overall, but the difference is usually small, hence the extra Bayesian advantages come at a small cost. Furthermore, since the MLSTMFCN is also a deep learning baseline, it would be interesting to see how it performs incorporated into a deep kernel, possibly used as a sequence-to-sequence transformation with the signature kernel on top. Obviously TS classification is a vast field and many other models could be considered; e.g.~we did not use recurrent GPs or GPSSMs since
\begin{enumerate*}[label=(\arabic*)]
\item 
  they have so-far not been used for TS classification, possibly because there is no sequential nature in the output space,%; there is no uncertainty about the outputs to propagate forward, therefore, they are probably expected to perform similar to a deterministic non-linear state space model with a GP layer on top.
\item we did not find a GPflow implementation that would allow to use sequence kernels in the GP transition function. (Implementation of~\cite{Ialongo2019Overcoming} does currently not allow taking subsequences of past states.
  An implementation would require much further work, but an interesting project would be to combine our models.) 
\end{enumerate*}

\begin{figure}[t] 
	\centering
	\begin{minipage}{0.55\textwidth}
		\hspace{-40pt}
		\centering
		\includegraphics[width=3.25in]{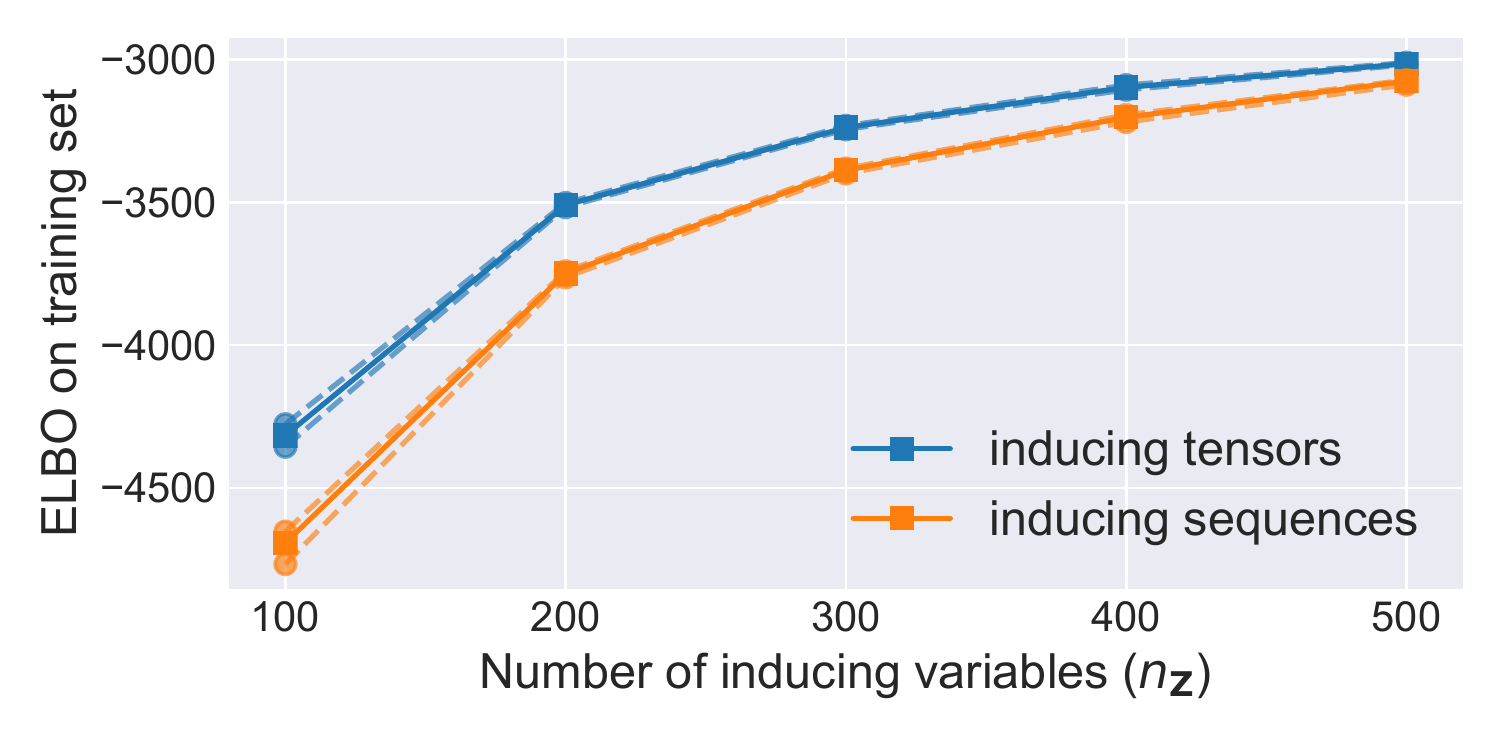}
		% \captionof{figure}{Caption for image}
		% \label{fig:sample_figure}
	\end{minipage}
	\begin{minipage}{0.55\textwidth}
		\hspace{-40pt}
		\centering
		\includegraphics[width=3.25in]{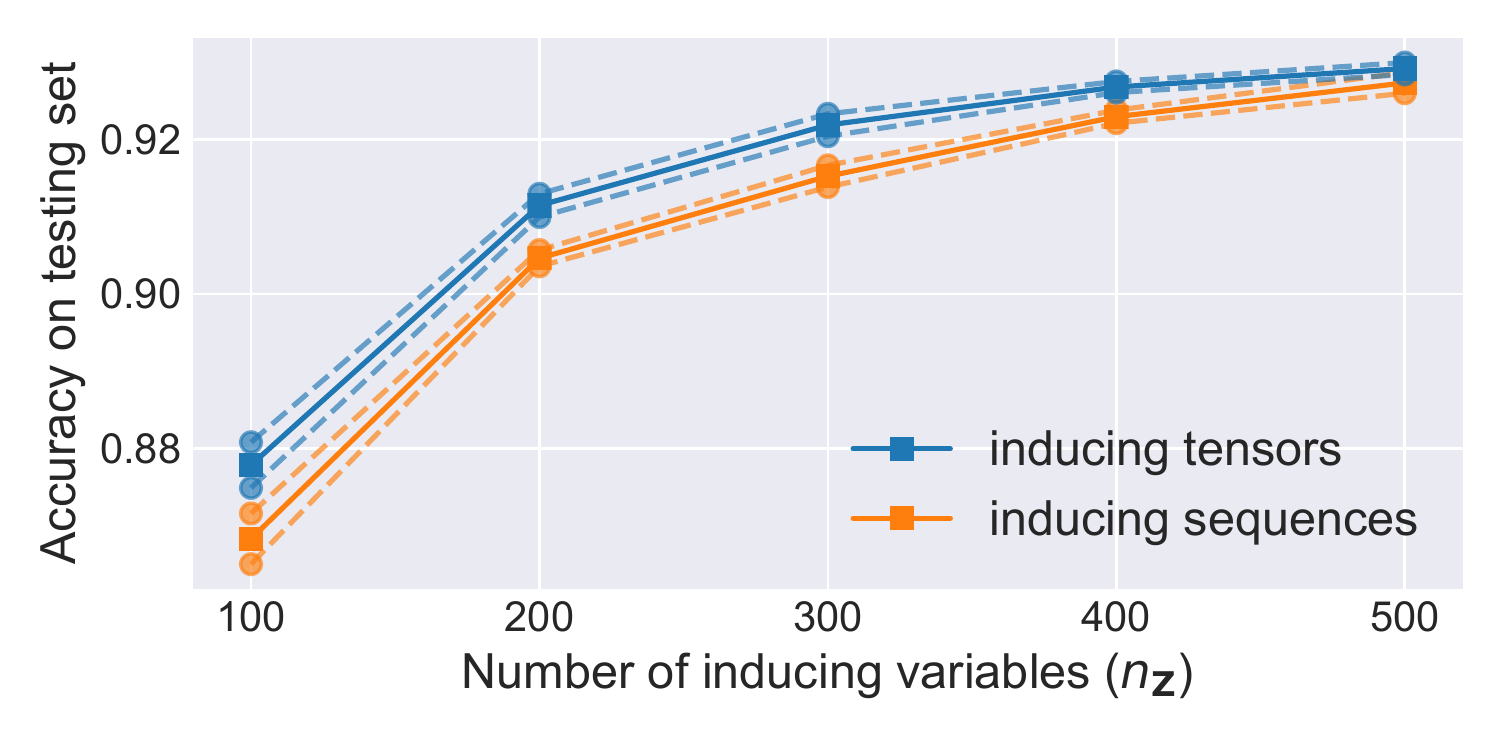}
		% \captionof{figure}{Caption for image}
		% \label{fig:sample_figure}
	\end{minipage}
	\begin{minipage}{0.55\textwidth}
	    \hspace{-40pt}
		\centering
		\includegraphics[width=3.25in]{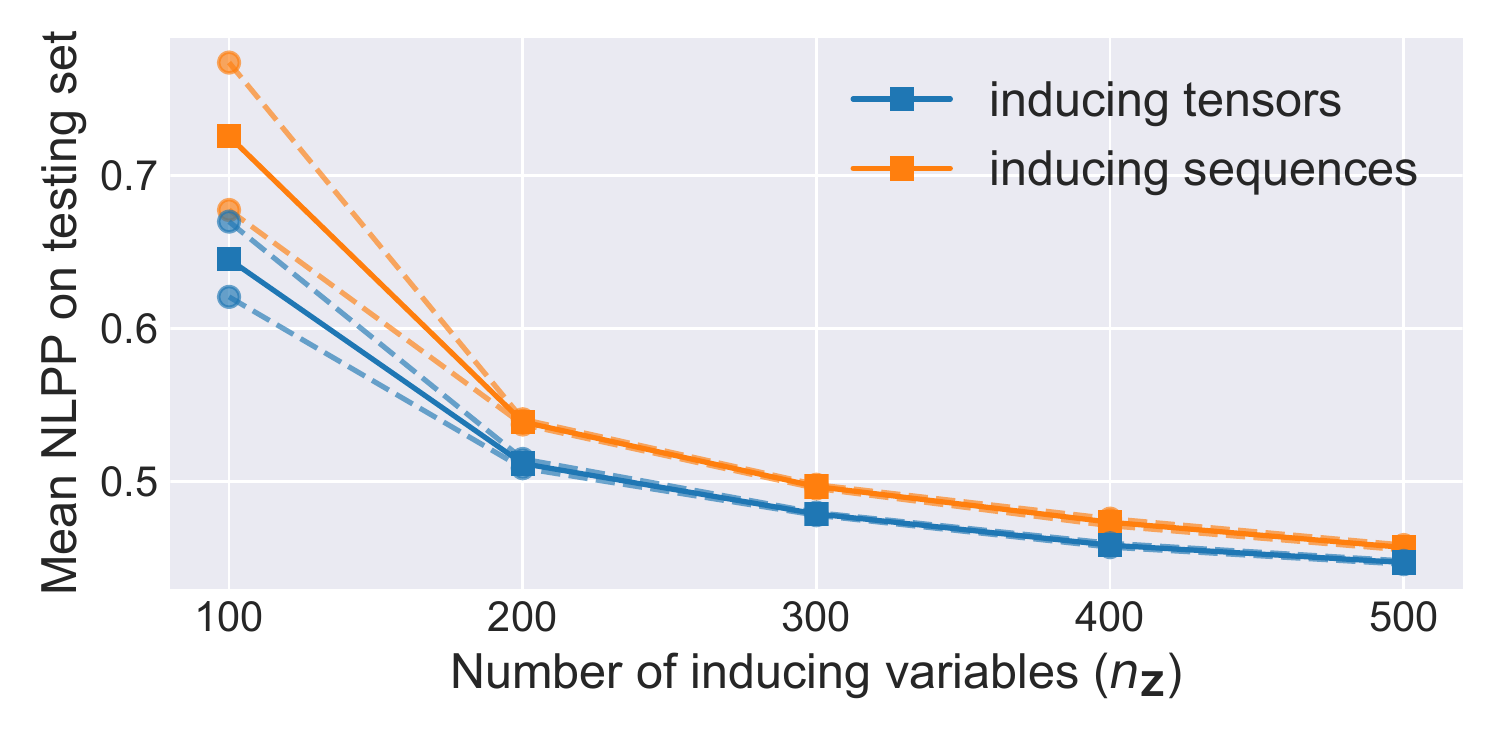}
		% \captionof{figure}{Caption for image}
		% \label{fig:sample_figure}
	\end{minipage}
% 	\vspace{-10pt}
	\caption{Achieved ELBO (top), accuracy (middle), mean nlpp (bottom) after $300$ epochs of training the variational parameters with random initialization and pre-learnt kernel hyperparameters, that were treated as fixed. Solid is the mean over $5$ independent runs, dashed is the $1$-std region.}
	%	\caption{ELBO on training set (left) and mean NLPP on holdout set (right) plotted against the the number of inducing variables used in the sparse approximation after training on the data for $300$ epochs with fixed kernel hyperparameters, that were learnt a-priori. In solid is the mean value and dashed is the one standard-deviation tube over 5 independent runs for each setting with random initialization of the variational parameters. }
	%	\label{fig:convergence}
	\label{fig:ind_tens_vs_ind_seq}
	\vspace{-10pt}
\end{figure}

\paragraph{Inducing tensors vs inducing sequences.} Our results rely on the inter-domain approach using tensors to locate inducing points from Section~\ref{sec:sparse var tensor}. %$ to reduce the time and memory complexities of computing $K_{\bZ\bZ}$ and $K_{\bX\bZ}$. 
An alternative is to use sequences for the inducing points, $\bZ \subset \cX_{seq}$, and controlling their maximal length $l_\bZ := \max_{\bz \in \bZ} l_{\bz}$ to be of order $m$, i.e. $l_{\bZ} \sim m$.
We coin this approach \textit{inducing sequences}.
Intuitively, one expects the inducing tensors to be more efficient than inducing sequences, since they make full use of the structure of the signature feature space/covariance.
To test this intuition, we compared the performance of the inducing tensors and inducing sequences subject to both having the same computational complexitiy.
For this experiment, we took the AUSLAN dataset~\cite{Dua2017UCI}, which consists of $n_c = 95$ classes for $n_\bX = 1140$ training examples.
This is a challenging dataset as the inducing variables need to characterize the abundance of classification boundaries.

We used GP-Sig with the same settings as in the previous experiments.
The hyperparameters of the kernel were a-priori learnt with $n_{\bZ} = 500$ inducing tensors, and we purely investigated how the quality of the approximation changes for both approaches by varying the number of inducing points $n_\bZ$.
For each number of inducing variables, both approaches were trained independently $5$ times for $300$ epochs with random initialization of the inducing variables, for details on which see Appendix~\ref{app:training}.
We plot on Figure \ref{fig:ind_tens_vs_ind_seq} three metrics: \begin{enumerate*}[label=(\arabic*)] \item the achieved ELBO on the training set, \item the achieved accuracy, and \item nlpp on the testing set \end{enumerate*}.
At $n_{\bZ} = 500$ both approaches are close to saturation, but the inducing tensors consistently perform better.
We remark that in practice, an important aspect is also how well the kernel hyperparameters can be recovered, that we did not consider here, and is a tricky question for sparse variational inference in general \cite{bauer2016understanding}.
Although, intuition suggests that the closer the model to saturation is with respect to the inducing points, the more consistent should the optimization be with un-sparsified variational inference.

\begin{figure}[t]
	\centering
	\begin{minipage}{0.55\textwidth}
        \hspace{-40pt}
        \centering
		\includegraphics[width=3.40in]{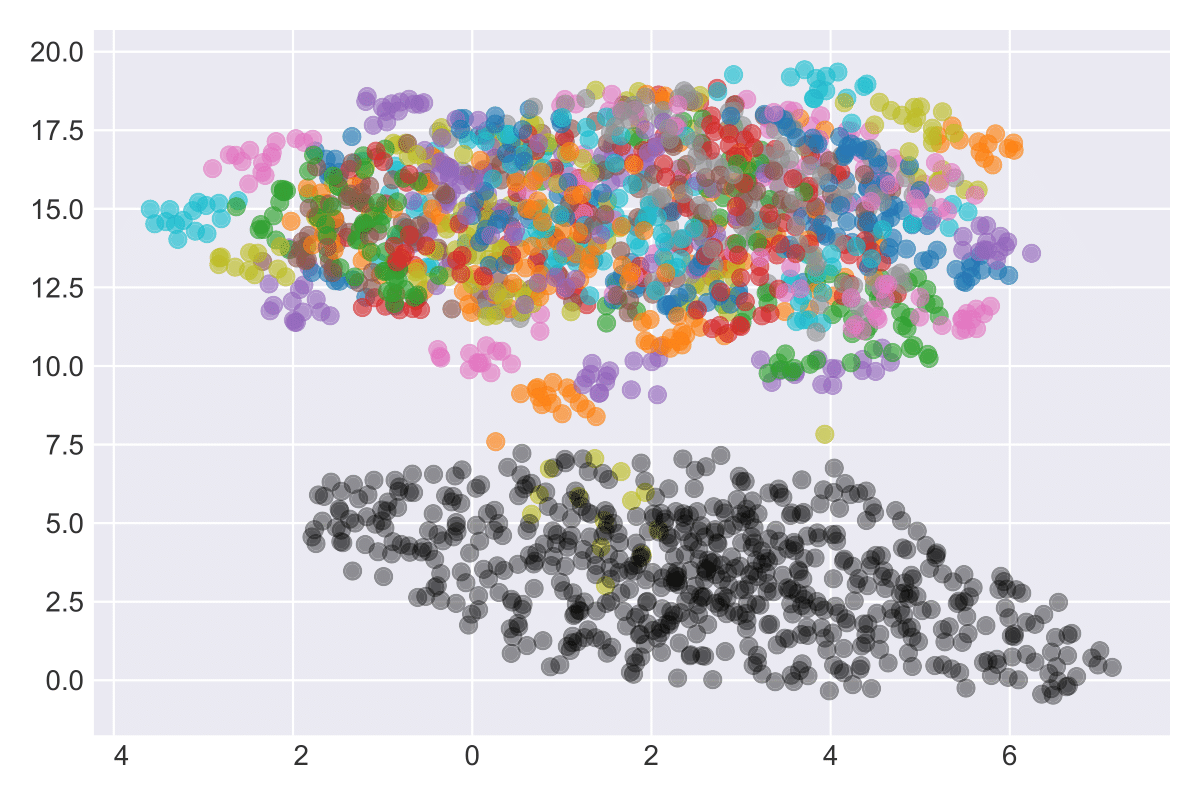}
		% \captionof{figure}{Caption for image}
		% \label{fig:sample_figure}
	\end{minipage} \hspace{-5pt}
	\vspace{-10pt}
	\caption{A UMAP visualization of the allocation of feature representations of data-points (coloured), and inducing tensors (black) in the feature space on the AUSLAN dataset.}
	%	\caption{ELBO on training set (left) and mean NLPP on holdout set (right) plotted against the the number of inducing variables used in the sparse approximation after training on the data for $300$ epochs with fixed kernel hyperparameters, that were learnt a-priori. In solid is the mean value and dashed is the one standard-deviation tube over 5 independent runs for each setting with random initialization of the variational parameters. }
	%	\label{fig:convergence}
	\label{fig:tens_umap}
	\vspace{-10pt}
\end{figure}

\paragraph{Visualizing inducing tensors.}
To gain more intuition, we visualized the feature space for one of the trained models on AUSLAN with $n_\bZ = 500$ inducing tensors.
We used UMAP~\cite{mcinnes2018umap-software} with the semi-metric $(\bx, \by) \mapsto (\kernel(\bx, \bx) + \kernel(\by, \by) - 2 \kernel(\bx, \by))^{0.5}$ for $\bx, \by \in \cX \cup \cX^\p$, see~Figure~\ref{fig:tens_umap}.
There are two imminent observations: \begin{enumerate*}[label=(\roman*)] \item in the point cloud corresponding to the data, the classes hardly look linearly separable; \label{umap:point1} \item the tensors, however, seem to live on a completely separate subspace than the data\label{umap:point2}.\end{enumerate*} 
The algorithm achieves $92\%$ accuracy on this set, therefore, point \ref{umap:point1} is likely due to information being lost in the projection. 
However, point \ref{umap:point2} challenges the intuition about classical sparse variational inference, that the inducing points are located mixed-in with the data-points, concentrating close to the classification boundaries \cite{Hensman2015Scalable}.
In general, the mechanism of how inter-domain inducing points represent the information in the data seems to be more complicated than classically.

To explain point \ref{umap:point2}, we remark that this phenomenon is not surprising at all: signature features live in a manifold that is embedded in the linear tensor space $\prod_{m=0}^M V^{\otimes m}$. 
In general, sparse tensors of the form~\eqref{eq:sparse_tens} will \emph{not} be signatures of paths. 
We believe variational inference works because of an interplay of two factors:
Firstly, signatures of finite sequences can be written as finite linear combinations of such sparse tensors. 
Secondly, the prior conditional term used to define $q(f_{\bx} \vert f_{\bZ})) = p(f_{\bx} \vert f_{\bZ})$.
The feature space is \emph{congruent} to the prior GP~\cite{berlinet2003reproducing}, which means that for $\bx \in \cX_{seq}$, the value of $f_\bx$ given $f_{\bZ}$ is not only almost deterministic when $\bx$ is close to any of $\bz \in \bZ$, but when it is close to any linear combinations of elements in $\bZ$.
By the first remark this can always be achieved given a large enough $n_\bZ$.
To sum up, the inducing tensors do not represent signature features individually, but form atomic building blocks such that their linear combinations induce the actual variational posterior at the data-examples.
%Nevertheless, the observation that tensors take the form of a data projection into a parallel subspace is quite intriguing.

\section{Conclusion}
We used a classical object from stochastic analysis -- signatures -- to define a GP for sequential data.
The GP inherits many of the theoretical guarantees that are known for signature features such as universality and parametrization invariance.
To make it scalable, we develop ``inducing tensors'' that exploit the structure of the feature space, inter-domain inducing points, and variational inference. 
Applied in a plain vanilla variational framework, this yields a classifier, GP-Sig, that is not only competitive in terms of nlpp with other GP models, but also with state-of-the-art frequentist TS classifiers in terms of accuracy alone. As one of our reviewers remarked, several datasets we consider have a strong signal-to-noise ratio, which makes it worthwhile to point out that even for such datasets, the alternative GP baselines suffer on at least some of them, while the proposed models are consistently able to learn on all datasets. This observation ties in to the universality property, and it suggests that GPs with signatures can be a good starting point when building Bayesian models on time series datasets.
%$Especially the latter seems encouraging since TS classification is a classical subject involving many different communities with very strong baselines.

We also demonstrate that signatures can be used as a building block in deep kernels to build larger GP models that leverage the benefits of both, RNNs and signatures. Interestingly, we find that the vanilla GP-Sig model outperforms the GP-Sig-RNNs for smaller datasets, conforming to the intuition that smaller sample sizes are detrimental for recurrent neural nets. To really get the best of both worlds, one could insert an additional model selection step, that specifies whether a parametric transformation is layer used before feeding the input into the kernel or not. Alternatively, it could also be possible to increase the flexibility of the sequential GP model while staying within a purely nonparametric framework using deep GPs \cite{damianou2013deep} by e.g.~applying a GP layer as observation-wise state-space embedding before the kernel computation. The inference framework of \cite{salimbeni19deep} for deep GPs could also come in handy when moving to datasets with lower signal-to-noise ratios, which can require GP models capable of handling not only epistemic (reducible) uncertainty, but aleatoric (irreducible) uncertainty in the data. It would also be interesting to see if such sequence kernels can be used to improve recurrent GP models \cite{Mattos2016Recurrent, Ialongo2019Overcoming} by incorporating sequential information into the GP transition function, that could potentially allow for a more efficient latent state representation.

\section*{Acknowledgments}
CT is supported by the "Mathematical Institute Award" by the University of Oxford, HO is supported by the EPSRC grant “Datasig” [EP/S026347/1], the Alan Turing Institute, and the Oxford-Man Institute.
CT and HO would like to thank the reviewers for helpful and constructive comments.
% \putbib
% \end{bibunit}
\bibliography{roughpaths}

\def\cprime{$'$}
\begin{thebibliography}{67}
\providecommand{\natexlab}[1]{#1}
\providecommand{\url}[1]{\texttt{#1}}
\expandafter\ifx\csname urlstyle\endcsname\relax
  \providecommand{\doi}[1]{doi: #1}\else
  \providecommand{\doi}{doi: \begingroup \urlstyle{rm}\Url}\fi

\bibitem[Adler \& Taylor(2009)Adler and Taylor]{adler2009random}
Adler, R. and Taylor, J.
\newblock \emph{Random Fields and Geometry}.
\newblock Springer Monographs in Mathematics. Springer New York, 2009.
\newblock ISBN 9780387481166.
\newblock URL \url{https://books.google.co.uk/books?id=R5BGvQ3ejloC}.

\bibitem[Al-Shedivat et~al.(2017)Al-Shedivat, Wilson, Saatchi, Hu, and
  Xing]{AlShedivat2017Recurrent}
Al-Shedivat, M., Wilson, A.~G., Saatchi, Y., Hu, Z., and Xing, E.~P.
\newblock Learning scalable deep kernels with recurrent structure.
\newblock \emph{J. Mach. Learn. Res.}, 18\penalty0 (1):\penalty0 2850–2886,
  January 2017.
\newblock ISSN 1532-4435.

\bibitem[Bauer et~al.(2016)Bauer, van~der Wilk, and
  Rasmussen]{bauer2016understanding}
Bauer, M., van~der Wilk, M., and Rasmussen, C.~E.
\newblock Understanding probabilistic sparse {G}aussian process approximations.
\newblock In \emph{Advances in neural information processing systems}, pp.\
  1533--1541, 2016.

\bibitem[Baydogan(2015)]{baydogan2015multivarate}
Baydogan, M.
\newblock \emph{Multivariate Time Series Classification Datasets}, 2015.
\newblock URL \url{http://mustafabaydogan.com}.
\newblock [Accessed: 2020-02-05].

\bibitem[Baydogan \& Runger(2015{\natexlab{a}})Baydogan and
  Runger]{baydogan2015learning}
Baydogan, M.~G. and Runger, G.
\newblock Learning a symbolic representation for multivariate time series
  classification.
\newblock \emph{Data Mining and Knowledge Discovery}, 29\penalty0 (2):\penalty0
  400--422, 2015{\natexlab{a}}.

\bibitem[Baydogan \& Runger(2015{\natexlab{b}})Baydogan and
  Runger]{Baydogan2015TimeSR}
Baydogan, M.~G. and Runger, G.~C.
\newblock Time series representation and similarity based on local
  autopatterns.
\newblock \emph{Data Mining and Knowledge Discovery}, 30:\penalty0 476--509,
  2015{\natexlab{b}}.

\bibitem[Berlinet \& Thomas-Agnan(2003)Berlinet and
  Thomas-Agnan]{berlinet2003reproducing}
Berlinet, A. and Thomas-Agnan, C.
\newblock \emph{Reproducing Kernel Hilbert Spaces in Probability and
  Statistics}.
\newblock Springer US, 2003.
\newblock ISBN 9781402076794.
\newblock URL \url{https://books.google.co.uk/books?id=v79sBNG34coC}.

\bibitem[Bui et~al.(2017)Bui, Yan, and Turner]{bui2016unifying}
Bui, T.~D., Yan, J., and Turner, R.~E.
\newblock A {U}nifying {F}ramework for {G}aussian {P}rocess {P}seudo-{P}oint
  {A}pproximations using {P}ower {E}xpectation {P}ropagation.
\newblock \emph{Journal of Machine Learning Research}, 18\penalty0
  (104):\penalty0 1--72, 2017.
\newblock URL \url{http://jmlr.org/papers/v18/16-603.html}.

\bibitem[Chai(2012)]{chai2012variational}
Chai, K. M.~A.
\newblock Variational multinomial logit {G}aussian process.
\newblock \emph{Journal of Machine Learning Research}, 13\penalty0
  (Jun):\penalty0 1745--1808, 2012.

\bibitem[Chen(1958)]{chen-58}
Chen, K.-T.
\newblock Integration of paths---a faithful representation of paths by
  non-commutative formal power series.
\newblock \emph{Trans. Amer. Math. Soc.}, 89:\penalty0 395--407, 1958.

\bibitem[{Chevyrev} \& {Kormilitzin}(2016){Chevyrev} and {Kormilitzin}]{CK16}
{Chevyrev}, I. and {Kormilitzin}, A.
\newblock {A Primer on the Signature Method in Machine Learning}.
\newblock \emph{ArXiv e-prints}, March 2016.

\bibitem[Chevyrev \& Oberhauser(2018)Chevyrev and
  Oberhauser]{ChevyrevOberhauser18}
Chevyrev, I. and Oberhauser, H.
\newblock Signature moments to characterize laws of stochastic processes.
\newblock \emph{arXiv preprint 1810.10971}, 2018.
\newblock URL \url{https://arxiv.org/abs/1810.10971}.

\bibitem[Chevyrev et~al.(2018)Chevyrev, Nanda, and
  Oberhauser]{chevyrev_nanda_oberhauser_2018}
Chevyrev, I., Nanda, V., and Oberhauser, H.
\newblock Persistence {P}aths and {S}ignature {F}eatures in {T}opological
  {D}ata {A}nalysis.
\newblock \emph{IEEE Transactions on Pattern Analysis and Machine
  Intelligence}, pp.\  1--1, 2018.
\newblock \doi{10.1109/tpami.2018.2885516}.

\bibitem[Cho et~al.(2014)Cho, {van Merrienboer}, Gulcehre, Bougares, Schwenk,
  and Bengio]{Cho2014}
Cho, K., {van Merrienboer}, B., Gulcehre, C., Bougares, F., Schwenk, H., and
  Bengio, Y.
\newblock Learning phrase representations using rnn encoder-decoder for
  statistical machine translation.
\newblock In \emph{Conference on Empirical Methods in Natural Language
  Processing (EMNLP 2014)}, 2014.

\bibitem[Chollet et~al.(2015)]{Chollet2015Keras}
Chollet, F. et~al.
\newblock Keras.
\newblock \url{https://github.com/fchollet/keras}, 2015.

\bibitem[Cuturi(2011)]{Cuturi2011GA}
Cuturi, M.
\newblock Fast global alignment kernels.
\newblock pp.\  929--936, 01 2011.

\bibitem[{Cuturi} \& {Doucet}(2011){Cuturi} and {Doucet}]{Cuturi2011AR}
{Cuturi}, M. and {Doucet}, A.
\newblock {Autoregressive Kernels For Time Series}.
\newblock \emph{arXiv e-prints}, art. arXiv:1101.0673, Jan 2011.

\bibitem[Damianou \& Lawrence(2013)Damianou and Lawrence]{damianou2013deep}
Damianou, A. and Lawrence, N.
\newblock Deep gaussian processes.
\newblock In \emph{Artificial Intelligence and Statistics}, pp.\  207--215,
  2013.

\bibitem[de~G.~Matthews et~al.(2017)de~G.~Matthews, van~der Wilk, Nickson,
  Fujii, Boukouvalas, Le{\'o}n-Villagr{\'a}, Ghahramani, and
  Hensman]{Matthews2017GPflowAG}
de~G.~Matthews, A.~G., van~der Wilk, M., Nickson, T., Fujii, K., Boukouvalas,
  A., Le{\'o}n-Villagr{\'a}, P., Ghahramani, Z., and Hensman, J.
\newblock Gpflow: A gaussian process library using tensorflow.
\newblock \emph{Journal of Machine Learning Research}, 18:\penalty0 40:1--40:6,
  2017.

\bibitem[Doerr et~al.(2018)Doerr, Daniel, Schiegg, Duy, Schaal, Toussaint, and
  Sebastian]{Doerr2018Proba}
Doerr, A., Daniel, C., Schiegg, M., Duy, N.-T., Schaal, S., Toussaint, M., and
  Sebastian, T.
\newblock Probabilistic recurrent state-space models.
\newblock In Dy, J. and Krause, A. (eds.), \emph{Proceedings of the 35th
  International Conference on Machine Learning}, volume~80 of \emph{Proceedings
  of Machine Learning Research}, pp.\  1280--1289, Stockholmsmässan, Stockholm
  Sweden, 10--15 Jul 2018. PMLR.

\bibitem[Dozat(2015)]{Dozat2015IncorporatingNM}
Dozat, T.
\newblock Incorporating {N}esterov {M}omentum into {A}dam.
\newblock 2015.

\bibitem[Dua \& Graff(2017)Dua and Graff]{Dua2017UCI}
Dua, D. and Graff, C.
\newblock {UCI} machine learning repository, 2017.
\newblock URL \url{http://archive.ics.uci.edu/ml}.

\bibitem[Dudley(2010)]{dudley2010sample}
Dudley, R.~M.
\newblock Sample functions of the gaussian process.
\newblock In \emph{Selected Works of RM Dudley}, pp.\  187--224. Springer,
  2010.

\bibitem[Eleftheriadis et~al.(2017)Eleftheriadis, Nicholson, Deisenroth, and
  Hensman]{Eleftheriadis2017Identification}
Eleftheriadis, S., Nicholson, T., Deisenroth, M., and Hensman, J.
\newblock Identification of gaussian process state space models.
\newblock In Guyon, I., Luxburg, U.~V., Bengio, S., Wallach, H., Fergus, R.,
  Vishwanathan, S., and Garnett, R. (eds.), \emph{Advances in Neural
  Information Processing Systems 30}, pp.\  5309--5319. Curran Associates,
  Inc., 2017.
\newblock URL
  \url{http://papers.nips.cc/paper/7115-identification-of-gaussian-process-state-space-models.pdf}.

\bibitem[Frigola et~al.(2013)Frigola, Lindsten, Sch\"{o}n, and
  Rasmussen]{Frigola2013MCMC}
Frigola, R., Lindsten, F., Sch\"{o}n, T.~B., and Rasmussen, C.~E.
\newblock Bayesian inference and learning in gaussian process state-space
  models with particle mcmc.
\newblock In Burges, C. J.~C., Bottou, L., Welling, M., Ghahramani, Z., and
  Weinberger, K.~Q. (eds.), \emph{Advances in Neural Information Processing
  Systems 26}, pp.\  3156--3164. Curran Associates, Inc., 2013.
\newblock URL
  \url{http://papers.nips.cc/paper/5085-bayesian-inference-and-learning-in-gaussian-process-state-space-models-with-particle-mcmc.pdf}.

\bibitem[Frigola et~al.(2014)Frigola, Chen, and
  Rasmussen]{Frigola2014Variational}
Frigola, R., Chen, Y., and Rasmussen, C.~E.
\newblock Variational gaussian process state-space models.
\newblock In \emph{Proceedings of the 27th International Conference on Neural
  Information Processing Systems - Volume 2}, NIPS’14, pp.\  3680–3688,
  Cambridge, MA, USA, 2014. MIT Press.

\bibitem[Ghahramani(2013)]{Ghahramani2013Bayesian}
Ghahramani, Z.
\newblock Bayesian non-parametrics and the probabilistic approach to modelling.
\newblock \emph{Philosophical transactions. Series A, Mathematical, physical,
  and engineering sciences}, 371:\penalty0 20110553, 02 2013.
\newblock \doi{10.1098/rsta.2011.0553}.

\bibitem[Glorot \& Bengio(2010)Glorot and Bengio]{Glorot10Understanding}
Glorot, X. and Bengio, Y.
\newblock Understanding the difficulty of training deep feedforward neural
  networks.
\newblock In Teh, Y.~W. and Titterington, M. (eds.), \emph{Proceedings of the
  Thirteenth International Conference on Artificial Intelligence and
  Statistics}, volume~9 of \emph{Proceedings of Machine Learning Research},
  pp.\  249--256, Chia Laguna Resort, Sardinia, Italy, 13--15 May 2010. PMLR.

\bibitem[Hambly \& Lyons(2010)Hambly and Lyons]{MR2630037}
Hambly, B. and Lyons, T.
\newblock Uniqueness for the {S}ignature of a path of bounded variation and the
  reduced path group.
\newblock \emph{Ann. of Math. (2)}, 171\penalty0 (1):\penalty0 109--167, 2010.
\newblock ISSN 0003-486X.
\newblock \doi{10.4007/annals.2010.171.109}.
\newblock URL \url{http://dx.doi.org/10.4007/annals.2010.171.109}.

\bibitem[Heather \& Chain(2016)Heather and Chain]{Heather2016DNA}
Heather, J.~M. and Chain, B.
\newblock The sequence of sequencers: The history of sequencing dna.
\newblock \emph{Genomics}, 107\penalty0 (1):\penalty0 1 -- 8, 2016.
\newblock ISSN 0888-7543.
\newblock \doi{https://doi.org/10.1016/j.ygeno.2015.11.003}.
\newblock URL
  \url{http://www.sciencedirect.com/science/article/pii/S0888754315300410}.

\bibitem[Hensman et~al.(2013)Hensman, Fusi, and
  Lawrence]{Hensman2013GaussianPF}
Hensman, J., Fusi, N., and Lawrence, N.~D.
\newblock Gaussian processes for big data.
\newblock \emph{CoRR}, abs/1309.6835, 2013.

\bibitem[Hensman et~al.(2015)Hensman, de~G.~Matthews, and
  Ghahramani]{Hensman2015Scalable}
Hensman, J., de~G.~Matthews, A.~G., and Ghahramani, Z.
\newblock Scalable variational gaussian process classification.
\newblock In Lebanon, G. and Vishwanathan, S. V.~N. (eds.), \emph{AISTATS},
  volume~38 of \emph{JMLR Workshop and Conference Proceedings}. JMLR.org, 2015.
\newblock URL
  \url{http://dblp.uni-trier.de/db/conf/aistats/aistats2015.html#HensmanMG15}.

\bibitem[Hensman et~al.(2016)Hensman, Durrande, and Solin]{Hensman2016Fourier}
Hensman, J., Durrande, N., and Solin, A.
\newblock Variational fourier features for gaussian processes, 2016.

\bibitem[Hochreiter \& Schmidhuber(1997)Hochreiter and
  Schmidhuber]{hochreiter1997long}
Hochreiter, S. and Schmidhuber, J.
\newblock Long short-term memory.
\newblock \emph{Neural computation}, 9\penalty0 (8):\penalty0 1735--1780, 1997.

\bibitem[Ialongo et~al.(2019)Ialongo, Van Der~Wilk, Hensman, and
  Rasmussen]{Ialongo2019Overcoming}
Ialongo, A.~D., Van Der~Wilk, M., Hensman, J., and Rasmussen, C.~E.
\newblock Overcoming mean-field approximations in recurrent {G}aussian process
  models.
\newblock In Chaudhuri, K. and Salakhutdinov, R. (eds.), \emph{Proceedings of
  the 36th International Conference on Machine Learning}, volume~97 of
  \emph{Proceedings of Machine Learning Research}, pp.\  2931--2940, Long
  Beach, California, USA, 09--15 Jun 2019. PMLR.
\newblock URL \url{http://proceedings.mlr.press/v97/ialongo19a.html}.

\bibitem[Ismail~Fawaz et~al.(2019)Ismail~Fawaz, Forestier, Weber, Idoumghar,
  and Muller]{Fawaz2019review}
Ismail~Fawaz, H., Forestier, G., Weber, J., Idoumghar, L., and Muller, P.-A.
\newblock Deep learning for time series classification: A review.
\newblock \emph{Data Min. Knowl. Discov.}, 33\penalty0 (4):\penalty0 917–963,
  July 2019.
\newblock ISSN 1384-5810.
\newblock \doi{10.1007/s10618-019-00619-1}.
\newblock URL \url{https://doi.org/10.1007/s10618-019-00619-1}.

\bibitem[Karim et~al.(2019)Karim, Majumdar, Darabi, and
  Harford]{Karim2019LSTMFCN}
Karim, F., Majumdar, S., Darabi, H., and Harford, S.
\newblock Multivariate lstm-fcns for time series classification.
\newblock \emph{Neural Networks}, 116:\penalty0 237 -- 245, 2019.
\newblock ISSN 0893-6080.
\newblock \doi{https://doi.org/10.1016/j.neunet.2019.04.014}.
\newblock URL
  \url{http://www.sciencedirect.com/science/article/pii/S0893608019301200}.

\bibitem[Karlsson et~al.(2016)Karlsson, Papapetrou, and
  Bostr\"{o}m]{karlsson2016generalized}
Karlsson, I., Papapetrou, P., and Bostr\"{o}m, H.
\newblock Generalized random shapelet forests.
\newblock \emph{Data Min. Knowl. Discov.}, 30\penalty0 (5):\penalty0
  1053–1085, September 2016.
\newblock ISSN 1384-5810.
\newblock \doi{10.1007/s10618-016-0473-y}.
\newblock URL \url{https://doi.org/10.1007/s10618-016-0473-y}.

\bibitem[Khoshnevisan(2002)]{khoshnevisan2002multiparameter}
Khoshnevisan, D.
\newblock \emph{Multiparameter processes: an introduction to random fields}.
\newblock Springer Science \& Business Media, 2002.

\bibitem[Kidger et~al.(2019)Kidger, Bonnier, Perez~Arribas, Salvi, and
  Lyons]{Kidger2019DeepSig}
Kidger, P., Bonnier, P., Perez~Arribas, I., Salvi, C., and Lyons, T.
\newblock Deep signature transforms.
\newblock In Wallach, H., Larochelle, H., Beygelzimer, A., d\textquotesingle
  Alch\'{e}-Buc, F., Fox, E., and Garnett, R. (eds.), \emph{Advances in Neural
  Information Processing Systems 32}, pp.\  3099--3109. Curran Associates,
  Inc., 2019.
\newblock URL
  \url{http://papers.nips.cc/paper/8574-deep-signature-transforms.pdf}.

\bibitem[Kingma \& Ba(2014)Kingma and Ba]{kingma2014adam}
Kingma, D.~P. and Ba, J.
\newblock Adam: A method for stochastic optimization, 2014.

\bibitem[Kiraly \& Oberhauser(2019)Kiraly and
  Oberhauser]{KiralyOberhauser2019KSig}
Kiraly, F.~J. and Oberhauser, H.
\newblock Kernels for sequentially ordered data.
\newblock \emph{Journal of Machine Learning Research}, 20\penalty0
  (31):\penalty0 1--45, 2019.
\newblock URL \url{http://jmlr.org/papers/v20/16-314.html}.

\bibitem[L{\'a}zaro-Gredilla \& Figueiras-Vidal(2009)L{\'a}zaro-Gredilla and
  Figueiras-Vidal]{lazaro2009inter}
L{\'a}zaro-Gredilla, M. and Figueiras-Vidal, A.
\newblock Inter-domain {G}aussian processes for sparse inference using inducing
  features.
\newblock In \emph{Advances in Neural Information Processing Systems}, pp.\
  1087--1095, 2009.

\bibitem[Li et~al.(2006)Li, Hastie, and Church]{Li2006VerySparse}
Li, P., Hastie, T., and Church, K.
\newblock Very sparse random projections.
\newblock volume 2006, pp.\  287--296, 01 2006.
\newblock \doi{10.1145/1150402.1150436}.

\bibitem[Lodhi et~al.(2002)Lodhi, Saunders, Shawe-Taylor, Cristianini, and
  Watkins]{lodhi2002text}
Lodhi, H., Saunders, C., Shawe-Taylor, J., Cristianini, N., and Watkins, C.
\newblock Text classification using string kernels.
\newblock \emph{The Journal of Machine Learning Research}, 2:\penalty0
  419--444, 2002.

\bibitem[Lyons \& Xu(2011)Lyons and Xu]{lyons2011inversion}
Lyons, T. and Xu, W.
\newblock Inversion of signature for paths of bounded variation, 2011.

\bibitem[Lyons et~al.(2007)Lyons, Caruana, and L{\'e}vy]{Lyons2007Differential}
Lyons, T.~J., Caruana, M., and L{\'e}vy, T.
\newblock Differential equations driven by rough paths, 2007.
\newblock Lectures from the 34th Summer School on Probability Theory held in
  Saint-Flour, July 6--24, 2004, With an introduction concerning the Summer
  School by Jean Picard.

\bibitem[Matthews(2017)]{MatthewsDPhil}
Matthews, A. G. d.~G.
\newblock \emph{Scalable {G}aussian process inference using variational
  methods}.
\newblock PhD thesis, Cambridge University, 2017.

\bibitem[Matthews et~al.(2016)Matthews, Hensman, Turner, and
  Ghahramani]{matthews2016sparse}
Matthews, A. G. d.~G., Hensman, J., Turner, R., and Ghahramani, Z.
\newblock On sparse variational methods and the {K}ullback-{L}eibler divergence
  between stochastic processes.
\newblock \emph{Journal of Machine Learning Research}, 51:\penalty0 231--239,
  2016.

\bibitem[Mattos et~al.(2016)Mattos, Dai, Damianou, Forth, Barreto, and
  Lawrence]{Mattos2016Recurrent}
Mattos, C. L.~C., Dai, Z., Damianou, A., Forth, J., Barreto, G.~A., and
  Lawrence, N.~D.
\newblock Recurrent {G}aussian processes.
\newblock In Larochelle, H., Kingsbury, B., and Bengio, S. (eds.),
  \emph{Proceedings of the International Conference on Learning
  Representations}, volume~3, Caribe Hotel, San Juan, PR, 00 2016.
\newblock URL
  \url{http://inverseprobability.com/publications/mattos-recurrent16.html}.

\bibitem[McInnes et~al.(2018)McInnes, Healy, Saul, and
  Grossberger]{mcinnes2018umap-software}
McInnes, L., Healy, J., Saul, N., and Grossberger, L.
\newblock Umap: Uniform manifold approximation and projection.
\newblock \emph{The Journal of Open Source Software}, 3\penalty0 (29):\penalty0
  861, 2018.

\bibitem[Micchelli et~al.(2006)Micchelli, Xu, and
  Zhang]{Micchelli2006Universal}
Micchelli, C.~A., Xu, Y., and Zhang, H.
\newblock Universal kernels.
\newblock \emph{Journal of Machine Learning Research}, 7\penalty0
  (Dec):\penalty0 2651--2667, 2006.

\bibitem[Oord et~al.(2016)Oord, Dieleman, Zen, Simonyan, Vinyals, Graves,
  Kalchbrenner, Senior, and Kavukcuoglu]{Oord2016Wavenet}
Oord, A. v.~d., Dieleman, S., Zen, H., Simonyan, K., Vinyals, O., Graves, A.,
  Kalchbrenner, N., Senior, A., and Kavukcuoglu, K.
\newblock Wavenet: A generative model for raw audio.
\newblock \emph{arXiv preprint arXiv:1609.03499}, 2016.

\bibitem[Pennington et~al.(2014)Pennington, Socher, and
  Manning]{Pennington2014Glove}
Pennington, J., Socher, R., and Manning, C.
\newblock {G}love: Global vectors for word representation.
\newblock In \emph{Proceedings of the 2014 Conference on Empirical Methods in
  Natural Language Processing ({EMNLP})}, pp.\  1532--1543, Doha, Qatar,
  October 2014. Association for Computational Linguistics.
\newblock \doi{10.3115/v1/D14-1162}.
\newblock URL \url{https://www.aclweb.org/anthology/D14-1162}.

\bibitem[Qui{\~n}onero-Candela \& Rasmussen(2005)Qui{\~n}onero-Candela and
  Rasmussen]{quinonero2005unifying}
Qui{\~n}onero-Candela, J. and Rasmussen, C.~E.
\newblock A unifying view of sparse approximate {G}aussian process regression.
\newblock \emph{Journal of Machine Learning Research}, 6\penalty0
  (Dec):\penalty0 1939--1959, 2005.

\bibitem[Rasmussen \& Williams(2006)Rasmussen and
  Williams]{Rasmussen2006Gaussian}
Rasmussen, C.~E. and Williams, C.~K.
\newblock Gaussian processes for machine learning. 2006.
\newblock \emph{The MIT Press, Cambridge, MA, USA}, 38:\penalty0 715--719,
  2006.

\bibitem[Salimbeni et~al.(2019)Salimbeni, Dutordoir, Hensman, and
  Deisenroth]{salimbeni19deep}
Salimbeni, H., Dutordoir, V., Hensman, J., and Deisenroth, M.
\newblock Deep {G}aussian processes with importance-weighted variational
  inference.
\newblock In Chaudhuri, K. and Salakhutdinov, R. (eds.), \emph{Proceedings of
  the 36th International Conference on Machine Learning}, volume~97 of
  \emph{Proceedings of Machine Learning Research}, pp.\  5589--5598, Long
  Beach, California, USA, 09--15 Jun 2019. PMLR.

\bibitem[Saxe et~al.(2014)Saxe, Mcclelland, and Ganguli]{Saxe2014Exact}
Saxe, A.~M., Mcclelland, J.~L., and Ganguli, S.
\newblock Exact solutions to the nonlinear dynamics of learning in deep linear
  neural network.
\newblock In \emph{In International Conference on Learning Representations},
  2014.

\bibitem[Sch{\"a}fer \& Leser(2017)Sch{\"a}fer and Leser]{Schfer2017MUSE}
Sch{\"a}fer, P. and Leser, U.
\newblock Multivariate time series classification with weasel+muse.
\newblock \emph{ArXiv}, abs/1711.11343, 2017.

\bibitem[Snelson \& Ghahramani(2006)Snelson and Ghahramani]{snelson2006sparse}
Snelson, E. and Ghahramani, Z.
\newblock Sparse {G}aussian processes using pseudo-inputs.
\newblock In \emph{Advances in neural information processing systems}, pp.\
  1257--1264, 2006.

\bibitem[Sriperumbudur et~al.(2011)Sriperumbudur, Fukumizu, and
  Lanckriet]{sriperumbudur2011universality}
Sriperumbudur, B.~K., Fukumizu, K., and Lanckriet, G.~R.
\newblock Universality, characteristic kernels and rkhs embedding of measures.
\newblock \emph{Journal of Machine Learning Research}, 12\penalty0
  (Jul):\penalty0 2389--2410, 2011.

\bibitem[Sutskever et~al.(2014)Sutskever, Vinyals, and
  Le]{Sutskever2014Seq2Seq}
Sutskever, I., Vinyals, O., and Le, Q.~V.
\newblock Sequence to sequence learning with neural networks.
\newblock In Ghahramani, Z., Welling, M., Cortes, C., Lawrence, N.~D., and
  Weinberger, K.~Q. (eds.), \emph{Advances in Neural Information Processing
  Systems 27}, pp.\  3104--3112. Curran Associates, Inc., 2014.
\newblock URL
  \url{http://papers.nips.cc/paper/5346-sequence-to-sequence-learning-with-neural-networks.pdf}.

\bibitem[Takens(1981)]{takens1981detecting}
Takens, F.
\newblock Detecting strange attractors in turbulence.
\newblock In \emph{Dynamical systems and turbulence, Warwick 1980}, pp.\
  366--381. Springer, 1981.

\bibitem[Titsias(2009)]{Titsias2009Variational}
Titsias, M.
\newblock Variational learning of inducing variables in sparse gaussian
  processes.
\newblock In van Dyk, D. and Welling, M. (eds.), \emph{Proceedings of the
  Twelth International Conference on Artificial Intelligence and Statistics},
  volume~5 of \emph{Proceedings of Machine Learning Research}, pp.\  567--574,
  Hilton Clearwater Beach Resort, Clearwater Beach, Florida USA, 16--18 Apr
  2009. PMLR.
\newblock URL \url{http://proceedings.mlr.press/v5/titsias09a.html}.

\bibitem[Tuncel \& Baydogan(2018)Tuncel and Baydogan]{tuncel2018autoregressive}
Tuncel, K.~S. and Baydogan, M.~G.
\newblock Autoregressive forests for multivariate time series modeling.
\newblock \emph{Pattern Recognition}, 73:\penalty0 202--215, 2018.

\bibitem[van~der Wilk et~al.(2017)van~der Wilk, Rasmussen, and
  Hensman]{Wilk2017ConvGP}
van~der Wilk, M., Rasmussen, C.~E., and Hensman, J.
\newblock Convolutional gaussian processes, 2017.

\bibitem[Wilson et~al.(2016)Wilson, Hu, Salakhutdinov, and
  Xing]{Wilson2016DeepK}
Wilson, A.~G., Hu, Z., Salakhutdinov, R., and Xing, E.~P.
\newblock Deep kernel learning.
\newblock In Gretton, A. and Robert, C.~C. (eds.), \emph{Proceedings of the
  19th International Conference on Artificial Intelligence and Statistics},
  volume~51 of \emph{Proceedings of Machine Learning Research}, pp.\  370--378,
  Cadiz, Spain, 09--11 May 2016. PMLR.

\end{thebibliography}
\bibliographystyle{icml2020}

% \begin{bibunit}
	
	\appendix
	% \section{Do \emph{not} have an appendix here}
	
	% \textbf{\emph{Do not put content after the references.}}
	% %
	% Put anything that you might normally include after the references in a separate
	% supplementary file.
	
	% We recommend that you build supplementary material in a separate document.
	% If you must create one PDF and cut it up, please be careful to use a tool that
	% doesn't alter the margins, and that doesn't aggressively rewrite the PDF file.
	% pdftk usually works fine. 
	
	% \textbf{Please do not use Apple's preview to cut off supplementary material.} In
	% previous years it has altered margins, and created headaches at the camera-ready
	% stage. 
	%%%%%%%%%%%%%%%%%%%%%%%%%%%%%%%%%%%%%%%%%%%%%%%%%%%%%%%%%%%%%%%%%%%%%%%%%%%%%%%
	%%%%%%%%%%%%%%%%%%%%%%%%%%%%%%%%%%%%%%%%%%%%%%%%%%%%%%%%%%%%%%%%%%%%%%%%%%%%%%%
	
	\newpage 
	\clearpage
	
    \textbf{\large Supplementary material}
	
	\section{Tensors} \label{app:tensors}
  We recall classical constructions with tensors.
  \subsection{Tensor products of vector spaces}
  If $u=(u_1,\ldots,u_d)\in \R^d$ and $v=(v_1,\ldots, v_e)\in \R^e$ then $u \otimes v \in \R^d \otimes \R^e $ is the $(d\times e)$-matrix with indices $i \in \{1,\ldots, d\}$, $j \in \{1,\ldots, e\}$ and the $(i,j)$-th entry given as $(u\otimes v)_{i,j}= u_i v_j$.
  Similarly, for $u \in \R^d$,$v \in \R^e$, $w \in \R^f$, the tensor $u \otimes v \otimes w \in \R^d \otimes \R^e \otimes \R^f$ has indices $i \in \{1,\ldots, d\}$, $j \in \{1,\ldots, e\}$, $j \in \{1,\ldots, f\}$ and its $(i,j,k)$-th entry is given as $(u \otimes v \otimes w)_{i,j,k} = u_i v_j w_k$, etc. 

  In the paragraph about variations in Section~\ref{sec:sparse var tensor}, we mention that one can also lift the sequence to a path evolving in an infinite-dimensional space $V$ rather than $\R^d$ before computing its signatures.
  Since $\int d x^{\otimes m} \in V^{\otimes m}$ this requires to take a tensor product of an infinite-dimensional space $V$.
  Since this might be less known in ML, let us briefly recall a coordinate-free definition of the tensor product:
  If $U$ and $V$ are vector spaces (not necessarily finite dimensional) then there exists a linear space $U\otimes V$ and a bilinear map $\iota: U \times V \rightarrow U \otimes V$ such that any other bilinear map on $U \times V$ factors through $U \otimes V$, that is given any bilinear map $B: U \times V \rightarrow Z$ into a vector space $Z$, there exists a linear map $\hat B: U \otimes V \rightarrow Z$ such that $\hat B \circ \iota = B $.
  Further, the vector space $U \otimes V$ is unique up to isomorphism.
  If $U,V$ are finite dimensional it is easy to verify that one recovers the coordinate-wise definition we recalled at the beginning of this section.
  If $U=V$ then we write $V^{\otimes 2}$ instead of $V \otimes V$; further by convention we define $V^{\otimes 0}:=\{1\}$.

\subsection{Sequences of tensors $\prod_{m=0}^M V^{\otimes m}$}
The direct product $\prod_{m\ge0} V_m$ of vector spaces $V_1,V_2,\ldots $ is the set of sequences 
\begin{align}
 \prod_{m\ge0} V_m :=\{(t_0,t_1,t_2,\ldots, ): t_m \in V_m\}.
\end{align}
In our setting, we apply this when $V$ is a vector space and $V_m:=V^{\otimes m}$, the get the space 
\begin{align}
  \prod_{m\ge0} V^{\otimes m}.
\end{align}
That is, an element $t =(t_m)_{m \ge 0} \in   \bigoplus_{m\ge0} V^{\otimes m}$ is a sequence of tensors of increasing depth, that is $t_0 = 1$ since by convention $V^{\otimes 0}= \{1\}$, $t_1 \in V$ is a vector, $t_2 \in V^{\otimes 2}$, etc.

The space $\prod_{m\ge0} V^{\otimes m}$ is itself a vector space if one defines addition and scalar multiplication coordinate-wise: for $s=(s_m)_{m \ge 0}$,$t=(t_m)_{m \ge 0} \in \prod_{m \ge 0} V^{\otimes m} $
\begin{align}
 s+t &= (s_0 + t_0, s_1 + t_1, s_2 + t_2,\ldots)\\ 
 \lambda \cdot s &= (\lambda s_0, \lambda s_1, \lambda s_2,\ldots)
\end{align}
That is, if $V=\R^d$ we add vectors to vectors, matrices to matrices, etc.
We note that the space $\prod_{m \ge 0} V^{\otimes m}$ is not just a vector space but has also a natural algebra structure and the space $\prod_{m \ge0} V^{\otimes m}$ is often referred to as the tensor algebra over $V$.
\subsection{Inner products of tensors}
We have seen how to build out of a linear space $V$ another linear space $\prod V^{\otimes m}$ of tensors.
If $V$ also carries an inner product, $\langle \cdot,\cdot \rangle_V$ this extends canonically to an inner product on subset of $\TA{V}$; set
\begin{align}
\langle v_1 \otimes \cdots \otimes v_m, w_1 \otimes \cdots \otimes w_m \rangle := \prod_{i=1}^m \langle v_i,w_i \rangle_V
\end{align}
and extend linearly to $\{t \in \TA{V}: \langle t,t \rangle < \infty \}$.
In particular, we can take linear functionals of $t$.

\subsection{Example: the classic polynomial features}
Take $\R^d$ with the standard Euclidean inner product.
In Section~\ref{sec:background} we recalled the classic ``polynomial feature map'' that takes a point in $\R^d$ to monomials in coordinates in $x$,
\begin{align}\label{eq: polynomial feature map}
\varphi: \bx \mapsto (\bx^{\otimes m})_{m \ge 0} \in \TA{(\R^d)}. 
\end{align}
We can build a functions $f: V \rightarrow \R$ by taking linear functionals of $\varphi$, that is for $\ell \in \prod_{m=0}^M (\R^d)^{\otimes m}$ define
\begin{align}\label{eq:pol feature}
 f: \bx \mapsto \langle  \ell, \varphi(\bx) \rangle  .
\end{align}
It might be helpful for readers less familiar with tensor products to spell out the definition of $f$ in coordinates:
by definition of the inner product we have 
\begin{align}
  \langle \ell, \Phi(x) \rangle= \sum_{m=1}^M \langle \ell_m, \bx^{\otimes m} \rangle.
\end{align}
Spelled out in coordinates, $\bx=(x_1,\ldots,x_d)$ and $\ell_m=(\ell_m^{i_1,\ldots,i_m})_{i_1,\ldots,i_m \in \{1,\ldots,d\}}$, the terms in the sum read as 
\begin{align}
\langle \ell_m, \bx^\otimes \rangle = \sum_{i_1,\ldots,i_m \in \{1,\ldots, d\}} \ell_m^{i_1,\ldots,i_m} x_{i_1}\cdots x_{i_m}.
\end{align}
Thus formulated in coordinates one has 
\begin{align}
f(x)&= \ell_0+\ell_1^1 x_1 +\cdots \ell_1^1 x_d \\&+ \ell_2^{1,1}x_1^2 + \ell_2^{1,2}  x_1 x_2 + \cdots + \ell_2^{2,2}x_d^2\\&+ \vdots \\&+ \ell_m^{1,\ldots,1} x_1^m +\cdots +\ell_m^{d,\ldots,d} x_d^m
\end{align}
which is how the polynomial feature map is often represented, see~\cite{Rasmussen2006Gaussian}.  
%Thus every $w \in \DS{V}\subset \TA{V}$ gives a linear functional $\langle w, \Sig(\bx) \rangle$ of the signature $\Sig(\bx)$.

  \section{Signature features}\label{app:signatures}
  In this Section we give background on signature features. 
  Signature features can  be seen as a natural generalization of the polynomial feature map, but instead of mapping a point in $\R^d$ to a sequence of tensors, they map paths $\cX_{path}$ to a sequence of tensors. 
  They generalize many of the nice properties of polynomial features such as universality and simulatenuously give the option to ignore the time-parametrization without an explicit search over all possible time changes (like in DTW approaches). 
	\subsection{Definition} \label{appendix:sig_properties}
	By Definition~\eqref{eq: signature}, the signature features are given as iterated integrals  
\begin{align}
 \textstyle{\Phi(\bx) = (1,\int_0^{t_\bx} d\bx, \int_0^{t_\bx} d\bx^{\otimes 2}, \ldots, \int_{0}^{t_\bx} d\bx_\tau^{\otimes m})} 
\end{align}
where $\textstyle{\int_0^{t} d\bx^{\otimes (m+1)}:= \int_0^{t_\bx} \int_0^{s}d\bx^{\otimes m} \otimes d\bx(s) \in (\R^d)^{\otimes m}}$ and by convention $\textstyle{\int_0^{t} d\bx^{\otimes 1}:= \bx(t)}$.
Hence, $\Phi(\bx) \in \prod_{m=0}^M (\R^d)^{\otimes m}$.
\subsection{Examples}
\paragraph{Coordinate-wise.}
For a path $x: t \mapsto (x_1(t), \ldots, x_d(t))$ that evolves in $\R^d$, one can spell this out in coordinates: the $m$-th signature feature $\int dx^{\otimes m} \in (\R^d)^{\otimes m}$ is the tensor that has as its $(i_1,\ldots, i_m) \in \{1,\ldots,d\}^m$-th coordinate entry the real number computed by a Riemann--Stieltjes integral
\begin{align*} \label{eq: sig coordinates}
\textstyle{\int dx_{i_1}(t_1) \cdots dx_{i_m}(t_m) = \int\dot x_{i_1}(t_1) \cdots \dot x_{i_m}(t_m) dt_1\cdots dt_m}
\end{align*}
where the integration is taken over ${0 \le t_1 < \cdots < t_m \le t_{\bx}}$ and $\dot x_i(t):= \frac{d x_i(t)}{dt}$.
\paragraph{Linear paths.}
Consider the path $x:[0,1] \rightarrow \R^d$ that just runs along a straight line
\begin{align}
  x(t)= t \mapsto tv 
\end{align}
where $v \in \R^d$ is a given vector.
Plugging~\eqref{eq: sig coordinates} into the definition of the iterated integrals, we get by a direct calculation that 
\begin{align}
 \int dx^{\otimes m}= \frac{v^{\otimes m}}{m!} \in (\R^d)^{\otimes m}.
\end{align}

We see that for this special case of a path $x$ that is fully described by its increment $x(t_{\bx}) - x(0) = v$, the signature features $\Phi(\bx)$ equal the polynomial features $\varphi(\bx(t_\bx) - \bx(0))$ of the total increment $v = \bx(t_\bx) - \bx(0)$ up to a rescaling by a constant $\frac{1}{m!}$. 
(This is one of the many reasons why signature features are regarded as ``polynomials of paths''). 

\paragraph{Piecewise linear paths.}
In general, these integrals need to be computed by standard integration techniques but for a piecewise linear path $\bx$, that is $[0,t]$ is partitioned into $L$ disjoint intervals, $[0,t_\bx] = \bigsqcup_{i=0}^{L-1} [t_i,t_{i+1}]$, and $\bx$ is piecewise linear on each of these pieces, $\bx(t)= t \cdot v_i $ for $t \in [t_i,t_{i+1}]$ for a vector $v_i \in \R^d$, then these iterated integrals just reduce to iterated sums, $\int dx^{\otimes m} $ equals 
\begin{align}
\textstyle{\sum_\bi c(\bi) v_{{i_1}}\otimes \cdots \otimes v_{{i_m}} \cdot (t_{i_1+1}- t_{t_1}) \cdots(t_{t_{i_m}+1}- t_{t_{i_m}})}
\end{align}
where the sum is taken over all tuples $\bi=(i_1,\ldots,i_m) \in \{1,\ldots, L\}^m$ and $c(\bi)$ is the inverse of the natural number $|\{p: \{1,\ldots, L\} \rightarrow \{1,\ldots,m\},\, p(i+1) \ge p(i), p(\bi)=\bi\}|!$.  
	
\subsection{Parametrization invariance.}
A classic result going back to Chen~\cite{chen-58} shows that the map $\bx \mapsto \Phi_0(\bx)$ is injective up to tree-like equivalence. 
Loosely speaking, tree-like equivalence is from a purely analytic point of view more natural to work with than reparametrization since tree-like equivalence between paths is analogous to Lebesgue almost sure equivalence between sets.
Howevever, we emphasize that from a practical point of view, the difference between paths that are tree-like equivalent and paths that differ by a reparametrization is negligible and we invite the reader to use them as synonyms throughout this article.
Nevertheless, we give the precise definition below and refer the interested reader to~\cite{MR2630037} for a detailed discussion. 
\begin{definition}
  A bounded variation path $\bx:[0,t_\bx] \rightarrow V$ is \emph{tree-like} if there exists a continuous function $h:[0,t_\bx] \rightarrow [0,\infty)$ such that $h(0)= h(T)=0$ and such that for all $s < t$
  \begin{align}
   | \bx(t) - \bx(s) | \le h(s) + h(t) - 2 \inf_{u \in [s,t]} h(u). 
  \end{align}
\end{definition}
\begin{theorem}
  Let $\bx:[0,t_\bx] \rightarrow V$ and $\by:[0,t_\by] \rightarrow V$ be two paths of bounded variation.
  Then 
\begin{align}
 \Phi(\bx) = \Phi(\by) 
\end{align}
if and only if $\bx \star \overleftarrow \by $ is tree-like where $\star$ denotes path concatenation and $\overleftarrow{\by}(t):= \by(t_\by - t)$ denotes time-reversal.
\end{theorem}
In particular this implies that for any function of the form 
\begin{align}
 f(\bx) = \langle \ell, \Phi_0 f(\bx) \rangle 
\end{align}
$f(\bx) = f(\by)$ if and only if $\bx$ and $\by$ differ by parametrization (strictly speaking, by a tree-like equivalence).
This ability to factor out time-invariance can be very powerful since the space of all possible time reparametrization is huge and we never make an explicit search over all possible time changes like in the calculation of DTW distance.

%In our signature features, this time-invariance is build in, so this can make the learning every efficient; in particular, one never does not need to perform a search over the space of all possible time changes as is done in DTW. 
\subsection{Parametrization variance.}
Often the functions of sequences $f(\bx)$ one is interested in, are invariant up to a certain degree of reparametrization but not invariant to extreme reparametrizations.
For a stylized example consider TS that arise as blood pressure measurements from patients responding to medication: some patients respond slower, some faster, depending on metabolism and many other factors. 
Up to a certain degree of time-reparameterisation one should observe a similar shaped TS if the medication works. 
However, the feature map should allow to distinguish extreme cases, e.g.~where the blood pressure is rapidly falling.

To address, we added an extra coordinate to the path $\bx$ before computing the signature features of this enhanced path $\bx_\tau(t)=(\tau \cdot t, \bx(t)) \in \R^{d+1}$, for $\tau > 0$.
The enhanced path $\bx_\tau$ is never tree-like since the first coordinate $t \mapsto t\cdot \tau$ is strictly increasing.
Formulated, differently: this ''trick`' makes the parametrization part of the trajectory. 
Hence, the map
\begin{align}
\textstyle{\cX_{paths}\ni \bx \mapsto \Phi_\tau(\bx):=\Phi(\bx_\tau) \in \prod_{m=0}^M (\R^{1+d})^{\otimes m} }
\end{align}
is injective for $\tau>0$.

\subsection{Universality.} \label{app:univ}
One of the most attractive properties of the classical polynomial feature map $\bx \rightarrow \varphi(\bx)$ for vectors $\bx \in \cX=\R^d$,~\eqref{eq: polynomial feature map}, is that any continuous function $f: \cX=\R^d \rightarrow \R$ can be uniformly approximated on compact sets as a linear functional of $\varphi$, that is $f(\bx) \approx \langle \ell, \varphi(\bx) \rangle$ for some $\ell$.
The reason is that linear combinations of monomials (polynomials) form an algebra and the Stone--Weierstrass theorem applies.
Such approximation properties of feature maps are usually referred to as ``universality'' in the ML literature.

One of the most attractive properties of the signature feature map $\bx \mapsto \Phi_\tau(\bx)$ for paths $\bx \in \cX_{paths}$ is that a universality result holds. 
For every continuous $f: \cX_{paths} \rightarrow \R$, $K\subset \cX_{paths}$ compact, $\epsilon>0$ there exists a $M\ge 1$, $\ell  \in \prod_{m = 0}^M V^{\otimes m}$ such that  
 	\begin{align}%\label{eq:universal}
	\label{eq:SW1}
\textstyle{\sup_{\bx \in K } | f(\bx) - \langle \ell, \Phi_{\tau}(\bx) \rangle | < \epsilon.}
	\end{align}
The analogous result holds for $\tau=0$ when we replace the domain $\cX_{paths}$ by equivalence classes of paths (under reparameterisation/tree-like equivalence). 
For a proof and many extensions, see~\cite{ChevyrevOberhauser18}.
\subsection{High-frequency sampling} 
One way to think about the embedding of $\cX_{seq} \hookrightarrow \cX_{paths}$ is that $\cX_{paths}$ represents the ``real-world'' where quantities evolve in continuous time but due to pratical reasons such as storage cost we only have access to their preimage in $\cX_{seq}$. 
%For example, if $\bx \in \cX_{paths}$ we sample it along a grid $\pi:=\{(t_1,\ldots,t_{\ell}):0\le t_1<\cdots < t_{\ell} \le t_\bx\}$ and record the measurements $x_i:=\bx(t_i)$ to produce a TS/sequence $\bx^{\pi} = (t_i,x_i)_{i=1,\ldots, t_\bx} \in \cX_{seq}$.
A natural question is what happens when the sampling gets finer and finer. 
We believe such consistency in the high-frequency sampling limit is important for the same reason, consistency in the number of samples $n_\bX$ is important: although in practice we only deal with finite numbers (finite number of samples, sequences rather than paths), we want that our method makes sense as we get more and more information.
In the context of learning with sequences this does not only require to study $n_\bX \rightarrow \infty$ but also the limit as the mesh size of that sampling grid converges to $0$. 

\paragraph{Consistency.}More formally, given $\bx \in \cX_{paths}$ consider a sequence $(\pi_k)$ of partitions
\begin{align}
\pi_{k}=\{(t^k_1,\ldots,t^k_n):0 \le t_1^k<\cdots < t_{n}^k \le t_{\bx}\}
\end{align}
with vanishing mesh
\begin{align}
 \operatorname{mesh}(\pi_k):= \max |t_{i+1}^k - t_i^k| \rightarrow 0 \text{ as }k \rightarrow \infty. 
\end{align}
Each partition $\pi_k$ gives rise to sequence $\bx^k$ by sampling $\gamma$ along the time points in $\pi_k$.
Following our convention we identify $\bx^k$ as a piecewise linear path in $\cX_{paths}$ and it is easy to verify that $\| \bx - \bx^k\| \rightarrow 0$ as $k \rightarrow \infty$.
Informally, as $k\rightarrow \infty$ we go from discrete to continuous time. 
One of the nice properties of our GP covariance, is that it is consistent under such limits: given $\bx,\by \in \cX_{paths}$, $k(\bx^k,\by^k) \rightarrow k(\bx,\by)$ as $k \rightarrow \infty$. 
Having a well-defined GP on paths that is consistent under such approximations from discrete to continuous time guarantee that no constants blow up as the sequences gets longer (sampling gets high frequent). 

\paragraph{Rough paths.}
So far we assumed that $\cX_{seq}$ consists of bounded variation paths but in the ``real-world'', the evolution of quantities is often subject to noise, e.g.~a classical model in physics and engineering is
\begin{align}
 \bx(t):= a(t) + B(t) 
\end{align}
where $a$ is a bounded variation path but $B$ is a Brownian sample path.
Since Brownian sample paths are not of bounded variation, $\bx$ is not of bounded variation.
However, the same consistency arguments as above go through but one has to replace the iterated Riemann--Stieltjes integrals by Ito--Stratonovich integrals in the definition of $\Phi(\bx)$.
Even rougher trajectories such as fractional Brownian motion and non-Markovian processes can be handled that way with so-called rough path integrals.
This is well-beyond the scope of the present article but we refer the interest reader to~\cite{ChevyrevOberhauser18} for such results. 

 \section{GPs with Signature Covariances} \label{app:gpsig}
 We specified a covariance function $k$ on the set $\cX_{paths}$ as inner product of the signature map.
 This guarantees that $(\bx,\by) \mapsto k(\bx,\by) = \langle  \Phi(\bx), \Phi(\by) \rangle$ is a positive definite function and from the general theory of stochastic processes the existence of a centered GP $(f_{\bx})_{\bx \in \cX_{paths}}$ such that $\E[ f_\bx f_\by] = k(\bx, \by)$ follows.  
 However, this does not guarantee that the sample paths $\bx \mapsto f_\bx$ are continuous. 
 Seminal work of Dudley~\cite{dudley2010sample} showed that such regularity estimates can be derived by bounding the growth of the covering number of the index set of the GP $f$ under the semi-metric
 \begin{align}
  d_k(\bx,\by) &= \sqrt{\E [ |f_\bx - f_\by|^2 ]} \\&= \sqrt{k(\bx,\bx) - 2 k(\bx,\by) + k(\by,\by)}. 
 \end{align}
 Already when when the index set is finite dimensional ``nice'' covariance functions can lead to discontinuous GPs, see e.g.~Section 1.4.~in \cite{adler2009random}. 
 Our GP has as index set the space of bounded variation paths $\cX_{paths}$ which is infinite-dimensional so some caution is needed. 
 However, as we show below we can cover this space by lattice paths and derive covering number estimates that imply continuity.
\begin{theorem}\label{thm:covering}
  For $L >0$ and $\epsilon>0$ denote with $N(\epsilon,L)$ the covering number of the set
  \[\cX_{paths}^L:=\{ x \in \cX_{paths}: \| x \|_{bv} \le L\}\] of bounded variation paths of length less or equal than $L$ under the $d_k$ pseudo-metric.
  Then
  \begin{align}
   \log_2 N(\epsilon,L) \le 2 (d+1) L\frac{\sqrt M}{\epsilon}
  \end{align}
  \end{theorem}
  \begin{proof}
    By definition of the metric $d_k$
    \begin{align}
    d_{k}(x,y) &\equiv \sqrt{\langle \Phi(\bx)-\Phi(\by),
                 \Phi(\bx)-\Phi(\by) \rangle} \\ &= \|\Phi(\bx)-\Phi(\by)\|.
    \end{align}
    By definition $\Phi$ and of the norm $\|\cdot\|$ on $\prod_{m=0}^M (\R^d)^{\otimes m}$ this reads 
    \begin{align}\label{eq: d_k}
d_k^2(\bx,\by) &= \sum_{m=1}^M \| \int d\bx^{\otimes m} - \int d\by^{\otimes m}\|^2  \\&\le M \max_{m=1,\ldots,M}\Delta_m^2(\bx,\by)
    \end{align}
   where we denote $\Delta_m(\bx,\by):=\| \int d\bx^{\otimes m} - \int d\by^{\otimes m}\|$.
   Let $\lattice{s}{L} \subset \cX_{paths}^L$ be the set of lattice paths starting at $0 \in \R^d$ that take steps of size $s$ and that are of total length at most $L$. 
By the results in Section 4 of~\cite{lyons2011inversion}, for every $\bx \in \cX_{paths}^L$ and every $n \ge 1$ there exists a $\by \in \lattice{L 2^{-n}}{L}$ such that for every $m \ge 1$, 
\begin{align}\label{eq:weijun}
\Delta_m(\bx,\by)\le \frac{d}{2^{n-1}} \frac{4 L^{m-1}}{(m-1)!}. 
\end{align}
Since $\frac{ L^{m-1}}{(m-1)!}\le e^L-1 $ we can apply \eqref{eq:weijun} with
$n=n(\epsilon):= 1-\log_2 \frac{{\epsilon} }{d\sqrt{M}4(e^L-1)}$ to get $\Delta_m(\bx,\by) \le \epsilon$.
Hence, there exists a lattice path $\by \in \lattice{L2^{-n(\epsilon)}}{L}$ such that
\begin{align}
 d_k(\bx,\by) \le \epsilon.
\end{align}
Further, the set $\cX_{paths}$ is finite and we can bound it by
\begin{align}
 |\lattice{L2^{-n(\epsilon)}}{L} | &\le (2^d+1)^{L2^{n(\epsilon)}} \le  2^{(d + 1) L2^{n(\epsilon)}} \\&=   2^{2 (d + 1)L2^{n(\epsilon)-1}} = 2^{2 (d+1)L\frac{\sqrt M}{\epsilon}} 
\end{align}
 where the first inequality follows since a lattice path has at every step $2^d$ directions to choose from and in addition can choose not to make a step.  
The last equality follows from the definition of $n(\epsilon)$.
Since $\bx\in \cX_{paths}^L$ was chosen arbitrary it follows that $\cX_{paths}^L$ can be covered by $2^{2(d+1)L\frac{\sqrt M}{\epsilon}}$ balls of radius $\epsilon$ centered at lattice paths.
\end{proof}
Theorem~\ref{thm:covering} combined with Dudley's celebrated entropy estimates gives regularity results for samples of our GP.
In fact, this even yields a modulus of continuity for our GP.
\begin{theorem}
  There exists a centered GP $(f_\bx)_{\bx \in \cX^L_{paths}}$ that has a covariance $\E[f_\bx f_\by]$ the signature covariance function $k(\bx,\by)=\langle \Phi(\bx), \Phi(\by) \rangle$.
  Moreover, if we denote its modulus of continuity on $\cX_{paths}^L$ with 
  \begin{align}
   \omega(\delta):= \sup_{\substack{\bx,\by \in \cX_{paths}^L\\ d_k(\bx,\by)< \delta}} |f_\bx - f_\by| 
  \end{align}
  then it holds with probability one that 
  \begin{align}\label{eq:modulus}
\limsup_{\delta \rightarrow 0} \frac{\omega(\delta) }{\sqrt {\delta}4\sqrt{ (d+1)L \sqrt{M}}  + c \delta \sqrt{\ln \ln \frac{1}{\delta}} } \le 24
  \end{align}
  where $c>0$ denotes a universal constant.
\end{theorem}
\begin{proof}
The existence of a centered GP $ (\hat f_\bx)_{\bx}$ with covariance $k$ follows from general results about Gaussian processes. 
The existence of a continuous modification $(f_\bx)_{\bx \in \cX_{paths,L}}$ of follows from Dudley's theorem if 
 \begin{align}
  \int_0^1 \sqrt{ \log_2 N (\epsilon,L)} d \epsilon  < \infty
 \end{align}
 but by Theorem~\ref{thm:covering} we have \[\int_0^1 \sqrt{ \log_2 N (\epsilon,L)} d \epsilon \le \sqrt {2 (d+1 )L \sqrt{M}} \int_0^1 \frac{1}{\sqrt \epsilon}d \epsilon < \infty.\]
 Dudley's results immediately yield a modulus of continuity in probability.
 By standard arguments this can be strengthened to give an almost sure modulus of continuity. 
 Concretely, we use the formulation given in Theorem 2.7.1 in Chapter 5 of \cite{khoshnevisan2002multiparameter} which guarantees that
 \[
 \limsup_{\delta \rightarrow 0} \frac{\omega(\delta) }{\int_0^\delta \sqrt{ N(\frac{\epsilon}{2},L)} d\epsilon + c \delta \sqrt{\ln \ln \frac{1}{\delta}} } \le 24.
 \]
The bound~\eqref{eq:modulus} follows immediately since first term in the denominator equals 
 \begin{align}
   \int_0^\delta \sqrt{ \log_2 N \left(\frac{\epsilon}{2},L\right)} d \epsilon  &= \sqrt {2(d+1)L \sqrt{M}}2 \sqrt{ 2}\sqrt{ \delta } \\&= \sqrt{\delta}4\sqrt{  (d+1)L \sqrt{M}}.
 \end{align}
\end{proof}

\section{Further algorithms}\label{app:algos}
  \subsection{Notation for computations.} \label{app:alg_notation}
  We define notation based on \cite{KiralyOberhauser2019KSig} for concisely describing vectorized computations. We use $1$-based indexing for arrays to keep in line with the notation of the main text. Let $A$ and $B$ be k-fold arrays of size $(n_1 \times \dots \times n_k)$, indexed by $i_j \in \{1, \dots, n_j\}$ for $j \in \{1, \dots, k\}$. We define the following operations.
  \begin{enumerate}[label=(\roman*)]
  	\item  The cumulative sum along axis $j$ as:
  	\begin{align}
  		&A[:, \dots, :, \boxplus, :, \dots, :][i_1, \dots, i_{j-1}, i_j, i_{j+1}, \dots i_k] \\ 
  		&\coloneqq \sum_{\kappa=1}^{i_j} A[i_1, \dots, i_{j-1}, \kappa, i_{j+1}, \dots, i_k].
  	\end{align}
  	\item The slice-wise sum along axis $j$ as:
  	\begin{align}
	  	&A[:, \dots, :, \Sigma, :, \dots, :][i_1, \dots, i_{j-1}, i_{j+1}, \dots, i_k] \\
	  	&\coloneqq \sum_{\kappa=1}^{n_j} A[i_1,\dots, i_{j-1}, \kappa, i_{j+1}, \dots i_k].
  	\end{align}
  	\item The shift along axis $j$ by $+m$ for $m \in \bbN$ as:
  	\begin{align}
	  	&A[:, \dots, :, +m, :, \dots, :][i_1, \dots, i_j, \dots, i_k] \\
	  	&\coloneqq \left\lbrace\begin{array}{ll} A[i_1,\dots, i_j-m, \dots i_k], & \text{ if } i_j > m, \\ 0, & \text{ if } i_j \leq m.  \end{array}\right.
  	\end{align}
  	\item The element-wise product of arrays $A$ and $B$ as:
  	\[A \odot B [i_1, \dots, i_k] \coloneqq A[i_1, \dots, i_k] \cdot B[i_1, \dots, i_k]. \] 
  \end{enumerate}

  Additionally, note that the use of the cumulative sum, $\boxplus$, in conjunction with the shift by $1$ operator, $+1$, along the same axis is equivalent to an exclusive cumulative sum, where in the new array the $i_j$th index contains the sum of the original array's elements from $1$ to $i_j-1$.
  
  \subsection{Covariances between sequences and sequences} \label{app:stream_covs}
  
  \begin{algorithm}[h]
  	\caption{Computing covariances at sequences, $K_{\bX\bX}$}
  	\label{alg:streams_cov}
  	\begin{algorithmic}[1]
  		\STATE {\bfseries Input:} Sequences $\bX=(\bx_i)_{i=1,\dots,n_{\bX}} \subset \cX_{seq}$, \\ scalars $(\sigma^2_0, \sigma^2_1, \dots, \sigma^2_M)$, depth $M \in \bbN$ 
  		\STATE Compute $K[i, j, l, k] \gets \langle \Delta x_{i, t_l}, \Delta x_{j, t_k} \rangle$ for $i, j \in \{1,\dots,n_\bX\}$, $l, k \in \{1,\dots,l_\bX\}$ 
  		\STATE Initialize $R[i, j] \gets \sigma_0^2$ for $i, j \in \{1,\dots,n_{\bX}\}$
  		\STATE Update $R \gets R + \sigma^2_1 \cdot K[:, :, \Sigma, \Sigma]$
  		\STATE Assign $A \gets K$
  		\FOR{$m=2$ {\bfseries to} $M$}
  		\STATE Iterate $A \gets K \odot A[:, :, \boxplus+1, \boxplus+1]$
  		\STATE Update $R \gets R + \sigma_n^2 \cdot A[:, :, \Sigma, \Sigma]$
  		\ENDFOR
  		\STATE {\bfseries Output:} Matrix of covariances $K_{\bX\bX} \gets R$
  	\end{algorithmic}
  \end{algorithm}

  We describe in Algorithm \ref{alg:streams_cov} the computation of the covariance matrix $K_{\bX\bX}$ of $n_\bX$ for sequences $\bX = (\bx_i)_{i=1,\dots,n_\bX} \subset \cX_{seq}$, which is a modification of Algorithm 3 from \cite{KiralyOberhauser2019KSig}. The observant reader will notice that for the vectorization a requirement is that all sequences in $\bX$ have the same length, $l_\bX := \sup_{\bx \in \bX} l_{\bx}$. In practice, this is only a computational restriction and can be circumvented by tabulating each sequence to be the same length, e.g. by repeating the last observation as required. The convenience of the parametrization invariance of signatures is that the results remain unchanged.
  
  Simple inspection says that the complexity of Algorithm \ref{alg:streams_cov} is of $O((M + d) \cdot n_\bX^2 \cdot l_\bX^2)$ in time and $O(d \cdot n_\bX \cdot l_\bX + n_\bX^2 \cdot l_\bX^2)$ in memory. Although, note that for factorizing likelihoods the computation of the ELBO and making inference about unseen examples $\bx_* \in \cX_{seq}$ with credible intervals only requires the diagonals of $K_{\bX\bX}$, i.e. $K_{\bX} := [k(\bx, \bx)]_{\bx \in \bX}$. Hence, for convenience, we give vectorized pseudo-code in Algorithm \ref{alg:streams_var} for computing $K_\bX$, which has complexities $O((M+c) \cdot n_\bX \cdot l_\bX^2)$ in time and $O(d \cdot n_\bX \cdot l_\bX + n_\bX \cdot l_\bX^2)$.
  
  \begin{algorithm}[h]
  	\caption{Computing variances at sequences, $K_\bX$}
  	\label{alg:streams_var}
  	\begin{algorithmic}[1]
  		\STATE {\bfseries Input:} Sequences $\bX=(\bx_i)_{i=1,\dots,n_{\bX}} \subset \cX_{seq}$, \\ scalars $(\sigma^2_0, \sigma^2_1, \dots, \sigma^2_M)$, depth $M \in \bbN$ 
  		\STATE Compute $K[i, l, k] \gets \langle \Delta x_{i, t_l}, \Delta x_{i, t_k} \rangle$ for $i \in \{1,\dots,n_\bX\}$, $l, k \in \{1,\dots,l_\bX\}$ 
  		\STATE Initialize $R[i] \gets \sigma_0^2$ for $i \in \{1,\dots,n_{\bX}\}$
  		\STATE Update $R \gets R + \sigma^2_1 \cdot K[:, \Sigma, \Sigma]$
  		\STATE Assign $A \gets K$
  		\FOR{$m=2$ {\bfseries to} $M$}
  		\STATE Iterate $A \gets K \odot A[:, \boxplus+1, \boxplus+1]$
  		\STATE Update $R \gets R + \sigma_n^2 \cdot A[:, \Sigma, \Sigma]$
  		\ENDFOR
  		\STATE {\bfseries Output:} Vector of variances $K_{\bX} \gets R$
  	\end{algorithmic}
  \end{algorithm}

\section{Further details on experiments}    

\subsection{Implementation details} \label{app:implementation}
The implementation of all considered GP models are available at \url{GITHUBAUTHOR}. Here, we detail the technicalities related to the implementation of each model.

\paragraph{GP-Sig.} This is the standard GP model with the signature kernel over sequences. This is built on top of GPflow \cite{Matthews2017GPflowAG}, and other than a few tweaks, they interface with GPflow models in a straightforward manner. Particularly for the kernel, there are several variants available with different state space embeddings, including RBF and Matérn static kernels. The hyperparameters of the kernel which are learnt from the data are: \begin{enumerate*}[label=(\arabic*)] \item the lengthscales corresponding to each state space dimension, \item the scaling parameters that multiply each signature level, allowing to strengthen or weaken its effect, \item the lag values by which the additional lagged versions of each coordinate are shifted, that is a continuous parameter and is applied using linear interpolation and flat extrapolation (i.e. when the queried time-point is negative then the value at time $0$ is used) \end{enumerate*}. In Section \ref{sec:experiments}, we denoted the augmented sequence with a time coordinate and $p$ lags by $\hat \bx := (t_i, x_{t_i}, x_{t_i - s_1}, \dots, x_{t_i - s_p})_{i=1, \dots, l_\bx}$. The lagged coordinates use the same lengthscales as the original ones, which in many cases leads to better generalization compared to not using lags (e.g. Takens' theorem \cite{takens1981detecting}).
The signature kernel is also normalized using the standard kernel normalization $\tilde \kernel(\bx, \by) := \kernel(\bx, \by) / \sqrt{\kernel(\bx, \bx) \kernel (\by, \by)}$, which we apply individually to each signature level. The supported inducing variables are \texttt{InducingTensors} and \texttt{InducingSequences} corresponding to the two variants described in the main text.

\paragraph{GP-Sig-LR.} As previously mentioned, there exists a low-rank variant of the signature kernel as introduced in \cite{KiralyOberhauser2019KSig}, which aims to approximate the feature map using a low-rank approximation, rather than computing inner product of signature features directly. Our implementation first uses the Nyström approximation to find a low-dimensional approximation of the state-space embedding, and then uses the primal formulation of the signature algorithms (see Algorithm 5 in \cite{KiralyOberhauser2019KSig}) to compute the signature kernel, while keeping the size of the low-rank factors manageable with sparse randomized projections \cite{Li2006VerySparse}. Its advantage is that it extends to very long time series due to linear complexity in the time series length $l_\bX \in \bbN$, while the quadratic complexity of the full-rank kernel needs to be addressed another way. We did not include this variant among the experiments because overall it performed much worse than the full-rank variant. There were two main issues: \begin{enumerate*}[label=(\roman*)] \item on several datasets it failed to fit the dataset due to being less flexible and noise, \item even when the predictive means are good, it can still give severely miscalibrated uncertainties similarly to classic kernel approximation techniques (Nyström, RFF), since an LR covariance matrix results in a degenerate GP prior. \end{enumerate*}  

\paragraph{GP(-Sig)-LSTM/GRU.} The RNN based models with a GP layer placed on top use the Keras implementation of the RNN architectures \cite{Chollet2015Keras}, while the GP parts use the GPflow API, which is possible as both packages can define the computational graph using the Tensorflow backend. However, since none of the packages supports the other, the resulting models have to be trained somewhat manually using the slightly more primitive Tensorflow API, and therefore are not very user friendly. It is up to future work to build a more user friendly API that makes it possible to deploy models that combine neural networks and sparse variational GPs in a convenient manner.

\paragraph{GP-KConv1D.} The $1$-dimensional convolutional kernel essentially uses the same code as \cite{Wilk2017ConvGP} included in the GPflow package, with some tweaks that allow different length time series to be compared by padding each sequence with nans and masking the nan entries during the computation. We also normalize the features corresponding to this kernel to unit length in the feature space using the standard kernel normalization. In the experiments, we set the window size to $w=10$, but a few datasets have $\min_{\bx \in \bX} l_\bx < 10$, and in those cases we set $w = \min\left(10, \min_{\bx \in \bX} l_\bx\right)$. Also, as the sequence length $l_\bx$, and hence, the number of windows can vary from instance to instance, the weighted version of the convolutional kernel from \cite{Wilk2017ConvGP} is not applicable in this case, and the translation invariant version is used.

\subsection{Datasets details} \label{app:datasets}

Table \ref{table:dataset_spec} details the datasets from \cite{baydogan2015multivarate} that we used for benchmarking. Here $c$ denotes the number of classes, $d$ the dimension of the sequence state space, $l_\bx$ the range of sequence lengths, $n_\bX$ and $n_{\bX_\star}$ respectively denote the number of examples in the pre-specified training and testing sets. In the experiments, all state space dimensions were normalized to zero mean and unit variance. For the models GP-Sig(-LSTM/GRU), GP-KConv1D, we subsampled very long time series to $l_\bX = 500$, in order to deal with the quadratic complexity of kernel evaluations and be able to fit within GPU memory limitations.

\begin{table}[t]
	\caption{Specification of datasets used for benchmarking}
	\label{table:dataset_spec}
    % \vskip 0.15in
    \begin{center}
    \begin{small}
    \begin{sc}
    \begin{tabular}{lrrrrrr}
    \toprule
    Dataset  & $c$ & $d$ & $l_\bx$ & $n_\bX$ & $n_{\bX_\star}$ \\
    \midrule
        Arabic Digits & 10 & 13 & \numrange[range-phrase = --]{4}{93} & 6600 & 2200\\
        AUSLAN & 95 & 22 & \numrange[range-phrase = --]{45}{136} & 1140 & 1425\\
        Char.~Traj.& 20 & 3 & \numrange[range-phrase = --]{109}{205} & 300 & 2558\\
        CMUsubject16 & 2 & 62 & \numrange[range-phrase = --]{127}{580} & 29 & 29\\
        DigitShapes & 4 & 2 & \numrange[range-phrase = --]{30}{98} & 24 & 16\\
        ECG & 2 & 2 & \numrange[range-phrase = --]{39}{152} & 100 & 100\\
        Jap.~Vowels & 9 & 12 & \numrange[range-phrase = --]{7}{29} & 270 & 370\\
        Kick vs Punch & 2 & 62 & \numrange[range-phrase = --]{274}{841} & 16 & 10\\
        LIBRAS & 15 & 2 & 45 & 180 & 180\\
        NetFlow & 2 & 4 & \numrange[range-phrase = --]{50}{997} & 803 & 534\\
        PEMS & 7 & 963 & 144 & 267 & 173\\
        PenDigits & 10 & 2 & 8 & 300 & 10692\\
        Shapes & 3 & 2 & \numrange[range-phrase = --]{52}{98} & 18 & 12\\
        UWave & 8 & 3 & 315 & 896 & 3582\\
        Wafer & 2 & 6 & \numrange[range-phrase = --]{104}{198} & 298 & 896\\
        Walk vs Run & 2 & 62 & \numrange[range-phrase = --]{128}{1918} & 28 & 16\\
    \bottomrule
    \end{tabular}
    \end{sc}
    \end{small}
    \end{center}
    % \vskip -1.0in
\end{table}

\subsection{Training details} \label{app:training}

\paragraph{Initialization.} For all models considered in the main text in Section \ref{sec:experiments}, the RBF kernel was used as static kernel, which has lengthscale parameters $(l_1, \dots, l_d)$, i.e. the RBF kernel over $\bbR^d$ is up to rescaling given by
\begin{align}
    \kappa(\bx, \bx^\p) := \exp\left(- \frac{1}{2} (\bx - \bx^\p)^\top \Sigma^{-1} (\bx - \bx^\p)\right)
\end{align}
with $\Sigma_{ii} := l_i^2$ a diagonal matrix. We used the initialization $l_{i}^{(0)} := \sqrt{\E[(x_i - x_i^\p)^2] \cdot d}$, where $x_i, x_i^\p$ are two independent copies of the $i$-th input space coordinate, and we used a stochastic estimator of this with typically $n=1000$ observation samples from the data. 

All considered models in Section \ref{sec:experiments} used some form of inducing variables.
For the signature models, they were placed in the feature space of the signature map in the form of inducing tensors.
These inducing tensors given in \eqref{eq:sparse_tens} are tensor products of elements in $V$.
As detailed at the end of Section \ref{sec:sparse var tensor}, although the state space of a sequence is $\R^d$, we can embed this sequence into a path that evolves in a linear space $V$ that does not have to be $\bbR^d$.
One way to do this is to use an observation-wise state space embedding given by a kernel $\kappa: \R^d \times \R^d \rightarrow \R$ and map a sequence $\bx=(t_i,x_i)$ to a sequence $\kappa_\bx = (t_i,\kappa_{x_i})$ that evolves in the RKHS $V$ of $\kappa$; here $\kappa_x:=\kappa(x,\cdot) \in V$.
Therefore signatures of depth $M$ now live in the space $\prod_{m=0}^M V^{\otimes m}$, which is for most kernels $\kappa$ a genuine infinite-dimensional space.
However, all computations from Sections \ref{sec:our GP} and \ref{sec:sparse var tensor} carry on mutatis mutandis, with the difference being that we do not have the flexibility to represent the inducing tensors as tensor products of arbitrary elements in $V$, which are generally infinite dimensional.
In this case, we take
\begin{align} \label{eq:sparse_tens_h0}
    \bz = (z_m)_{m=0, \dots, M} \in \prod_{m=0}^M V_0^{\otimes m}
\end{align}

with $z_0 \in \bbR$ and $z_m = \kappa(x_{m, 1}, \cdot) \otimes \dots \otimes \kappa(x_{m, m}, \cdot)$ with $x_{i, j} \in \bbR^d$ for $1 \leq j \leq i$, $1 \leq i \leq m$, $1 \leq m \leq M$, where $V_0 := \{\kappa(x, \cdot) : x \in \bbR^d\}$.
Put differently, the inducing tensors are also constrained to being tensor products of only such elements in $V$ which arise as reproducing kernels\footnote{The reproducing kernel associated to a point $x \in \cX$ is simply the kernel function evaluated in one of its arguments at $x$, i.e.~$\kappa(x, \cdot) \in \cH_0 \subset \cH$ for a kernel $\kappa: \cX \times \cX \rightarrow \bbR$.}  associated to vectors in $\bbR^d$.
Hence, the complexity of evaluating $\langle \kappa(x, \cdot), \kappa(x^\p, \cdot) \rangle$ is the same as in $\bbR^d$, and storing an element $\kappa(x, \cdot) \in V_0$ is the same memory.
Now, the initialization of the inducing tensors is simply done by sampling random observations from the input sequences in a two step manner: \begin{enumerate*}[label=(\arabic*)] \item a random input sequence is selected, \item from the sequence a time-increasing subset of its observations are selected and plugged into the tensor products given in \eqref{eq:sparse_tens_h0} \end{enumerate*}, and this procedure is repeated $n_\bZ$ times.

Other forms of inducing variables used by the models in Section \ref{sec:experiments} are inducing points for the GP-RNNs and inducing patches for GP-KConv1D. The inducing points are initialized randomly by selecting a $\bx \in \bX$ and computing its RNN-image $\phi_\theta(\bx)$, which is then used as an inducing point, and repeated for all $n_\bZ$. The inducing patches are also initialized in two steps: \begin{enumerate*}[label=(\arabic*)] \item select a random input sequence $\bx \in \bX$, \item select a random window from $\bx$, $(x_i, x_{i+1}, \dots, x_{i+w-1})$, where $1 \leq w \leq \min_{\bx \in \bX} l_\bx$ denotes the window length in the convolutional kernel. \end{enumerate*}

For the alternative sparse inference scheme for signatures described in Section \ref{sec:experiments}, denoted the method of inducing sequences, we use the same initialization as for the inducing patches: select a random sequence, and select a random window, and repeat for all $n_\bZ$.

The means and covariances of the inducing points used the usual whitening transformation, that is, reparametrization in terms of the Cholesky factor $L$ of $K_{\bZ\bZ}$, $K_{\bZ\bZ} = LL^\top$, and parameters initialized from zeros and identity.

The RNNs use the usual initializations, that is, Glorot initialization for the weights \cite{Glorot10Understanding}, orthogonal initialization for the recurrent weights \cite{Saxe2014Exact}, and zeros for the bias.

\paragraph{Optimization details.} The training for the benchmarking experiment in Section \ref{sec:sparse var tensor} was performed on 11 GPUs overall: 4 Tesla K40Ms, 5 Geforce 2080 TIs and 2 Quadro GP100 graphics cards. All models were trained 5 times for the benchmarking and the RNN based models an additional 6 times for the grid-search. Thus, the training of overall 480 models required extensive computational resources.

In all experiments in Section \ref{sec:experiments}, we used similar optimization details, that is, optimization with early stopping and checkpointing by optimizing on $80\%$ of the training data and monitoring the nlpp\footnote{We found that monitoring the validation nlpp rather than the validation accuracy leads to better generalization behaviour.} on a $20\%$ validation set. We used a minibatch size of $50$, fixed learning rate $\alpha = 1 \times 10^{-3}$, and a patience value of $n=500$ epochs.  As optimizer, GP-Sig and GP-KConv1D used Nadam \cite{Dozat2015IncorporatingNM}, while the RNN based models used Adam \cite{kingma2014adam}. Additionally, as is well-known for SVGPs \cite{bauer2016understanding}, first fixing the hyperparameters and only optimizing over the variational approximation for a fixed number of epochs is beneficial which we follow. Furthermore, after the main training phase of the hyperparameters has finished, to learn the rest of the validation data that was excluded from the optimization, we re-merge the validation set into the training set, fix the hyperparameters, and optimize only over the variational parameters again to assimilate the remaining information into the variational approximation.

Hence, the training for all models is split into the following phases \begin{enumerate*}[label=(\arabic*)] \item partition the data in an $80-20$ ratio for optimization and monitoring, \item with fixed kernel hyperparameters initialized as described previously, train the variational parameters for fixed $n$ epochs to tighten the ELBO bound; \item \label{enum:main_training_phase}  unfix the hyperparameters and train by monitoring the nlpp on the validation set, stopping after no improvement for $n$ epochs, and restoring and best model; \item re-merge the validation set into the training data and train the variational distribution again only for a fixed $n$ epochs with the kernel hyperparameters fixed \end{enumerate*}. In all scenarios, we used $n = 500$.

For GP-Sig, the insertion of an additional optimization phase was found to be beneficial. Particularly, we reparametrize the scaling parameters for the signature levels $\sigma = (\sigma_0, \dots, \sigma_M)$ as $\sigma = (\beta \cdot \sigma_0^\p, \dots, \beta \cdot \sigma_M^\p)$, where $\beta \in \bbR^+$. Then, phase \ref{enum:main_training_phase} is split into two steps: first, train with unfixing all kernel hyperparameters except $(\sigma_0^\p, \dots \sigma_M^\p)$, which are a-priori all set as $1$; secondly, now unfixing \emph{all} parameters, continue training with early stopping. This trick allows to calibrate the overall variance of the GP using $\beta$ in the first step, while fixing $\sigma_0 = \dots = \sigma_M$. The intution why this works is that the signature levels in general contain complementary information about a given sequence, and fixing them to be equal first enforces the model to find a fit of the data for all signature levels jointly, i.e. in some sense this is an implicit regularization step. The second step allows to slightly adjust the contribution of each level without relying too heavily on any one of them. On the RNN-based signature models this trick did not give substantial improvements, possibly because the variance of the RNN layer generally outweights the variance of the signature layer.

In our experience, when using GP-Sig on datasets with a larger $n_\bX$, it can yield a further improvement to gradually increase the learning to rate to $\alpha = 1 \times 10^{-2}$ to allow the optimizer to explore the space in more depth, and then decrease it back to $\alpha = 1 \times 10^{-3}$ to drive it to the closest local optima. However, on the smaller datasets this was found to be counterproductive, and in the experiments we chose to stick with a unified scheme that worked consistently on all datasets. However, we also remark that without applying any of the previously described techniques, and training from front to back all parameters jointly with a small learning rate (e.g. $\alpha = 1 \times 10^{-3}$) gives good results already, but a few percents of test set accuracy can be gained on some datasets by using them.

\begin{table*}[tb]
	\caption{List of architectures used for the RNN based models}
	\label{table:architecture_spec}
	% \vskip 0.15in
	\begin{center}
		\begin{small}
			\begin{sc}
				% \begin{tabular}{ c<{\ GHz} r @{\ $\angle$\ } r }
				\begin{tabular}{l  >{\ $H=$\ }l >{\ $D=$\ }l >{\ $H=$\ }l >{\ $D=$\ }l >{\ $H=$\ }l >{\ $D=$\ }l >{\ $H=$\ }l >{\ $D=$\ }l >{\ $H=$\ }l >{\ $D=$\ }l}
					% \begin{tabular}{ c<{\ GHz} r @{\ $\angle$\ } r } 
					\toprule
					\multicolumn{1}{l}{Dataset} & \multicolumn{2}{c}{GP-Sig-LSTM} & \multicolumn{2}{c}{GP-Sig-GRU} &
					\multicolumn{2}{c}{GP-LSTM} & \multicolumn{2}{c}{GP-GRU}\\
					%   & GP-Sig-LSTM & GP-Sig-GRU & GP-LSTM & GP-GRU \\
					\midrule 
                    Arabic Digits & $128$ & $1$ & $128$ & $1$ & $32$ & $1$ & $128$ & $1$ \\ 
                    AUSLAN & $128$ & $0$ & $128$ & $0$ & $128$ & $0$ & $32$ & $0$ \\ 
                    Character Traj. & $8$ & $1$ & $128$ & $1$ & $128$ & $1$ & $128$ & $1$ \\ 
                    CMUsubject16 & $32$ & $1$ & $32$ & $1$ & $32$ & $1$ & $32$ & $1$ \\ 
                    DigitShapes & $8$ & $1$ & $128$ & $1$ & $128$ & $1$ & $32$ & $1$ \\ 
                    ECG & $128$ & $0$ & $128$ & $1$ & $128$ & $0$ & $8$ & $1$ \\ 
                    Jap.~Vowels & $128$ & $0$ & $128$ & $1$ & $128$ & $1$ & $128$ & $0$ \\ 
                    Kick vs Punch & $8$ & $1$ & $8$ & $0$ & $128$ & $1$ & $128$ & $1$ \\ 
                    LIBRAS & $128$ & $0$ & $128$ & $0$ & $32$ & $0$ & $32$ & $0$ \\ 
                    NetFlow & $8$ & $0$ & $32$ & $0$ & $32$ & $0$ & $8$ & $1$ \\ 
                    PEMS & $32$ & $0$ & $8$ & $1$ & $32$ & $1$ & $32$ & $0$ \\ 
                    PenDigits & $128$ & $0$ & $128$ & $1$ & $128$ & $1$ & $128$ & $1$ \\ 
                    Shapes & $8$ & $1$ & $8$ & $1$ & $8$ & $1$ & $8$ & $1$ \\ 
                    UWave & $32$ & $1$ & $128$ & $0$ & $32$ & $0$ & $32$ & $0$ \\ 
                    Wafer & $128$ & $0$ & $128$ & $1$ & $32$ & $0$ & $32$ & $0$ \\ 
                    Walk vs Run & $128$ & $1$ & $8$ & $1$ & $8$ & $1$ & $32$ & $1$ \\
					\bottomrule
				\end{tabular}
			\end{sc}
		\end{small}
	\end{center}
	% \vskip -1.0in
\end{table*}

\paragraph{Architecture search.}
Table \ref{table:architecture_spec} details each of the architectures used for the models containing an RNN layer, where $H$ denotes the number of hidden units used, and $D$ is a boolean trigger, that specifies whether dropout was used for the given experiment or not. In the case $D=1$, we used the settings $\texttt{dropout = 0.25}$ and $\texttt{recurrent\_dropout = 0.05}$, otherwise both were set to $0$. To find the best performing architecture, we conducted a grid-search among $6$ considered architectures, that is, $H \in [8, 32, 128]$ and $D \in [0,1]$. For the grid-search, only the training data was used, and the data was split in a $60-20-20$ fashion, using $60\%$ for training, $20\%$ for early stopping and checkpointing, and the last $20\%$ was used to evaluate the performance. The training itself was carried out using the same initialization and schedule as described, and was performed only once for each method and setting pair, due to the large number of datasets that we considered.

\subsection{Benchmark results} \label{app:benchmark}

We report in Table \ref{table:full_nlpp_results} and Table \ref{table:full_acc_results} the negative log-predictive probabilities and accuracies of the GP models considered in Section \ref{sec:experiments}. For each method-dataset pair, 5 models were trained with the initialization described in Appendix \ref{app:training}. The variance of the results is therefore due to random initialization of some parameters, and the minibatch randomness while training. The RNN based models used the architectures detailed in Table \ref{table:architecture_spec}. As non-Bayesian baselines, we report the results of recent frequentist TS classification methods from the respective publications, that is, \cite{Cuturi2011AR, baydogan2015learning, Baydogan2015TimeSR, karlsson2016generalized, tuncel2018autoregressive, Schfer2017MUSE, Karim2019LSTMFCN}. Particularly for MLSTMFCN, we report the same results as in \cite{Schfer2017MUSE}. In Figure \ref{fig:boxplots}, we visualize the box-plot distributions of \begin{enumerate*}[label=(\arabic*)] \item negative log-predictive probabilities of the GPs, \item classification accuracies of both the GPs and the frequentist baselines \end{enumerate*}.

\begin{table*}[ht]
	\caption{Mean and standard deviation of negative predictive log-probabilities (nlpp) on test sets over $5$ independent runs}
	\label{table:full_nlpp_results}
	\vskip 0.15in
	\begin{center}
		\begin{small}
			\begin{sc}
				\begin{tabular}{lcccccc}%{lllll}
					\toprule
					Dataset & GP-Sig-LSTM & GP-Sig-GRU & GP-Sig & GP-LSTM  & GP-GRU & GP-KConv1D\\
					\midrule
                        Arabic Digits & $0.047 \pm 0.030$ & $0.023 \pm 0.006$ & $0.071 \pm 0.021$ & $0.082 \pm 0.022$ & $0.066 \pm 0.010$ & $0.050 \pm 0.003$ \\ 
                        AUSLAN & $0.106 \pm 0.007$ & $0.123 \pm 0.045$ & $0.550 \pm 0.114$ & $0.650 \pm 0.071$ & $0.248 \pm 0.063$ & $1.900 \pm 0.139$ \\ 
                        Character Traj. & $0.031 \pm 0.007$ & $0.258 \pm 0.265$ & $0.108 \pm 0.005$ & $2.506 \pm 1.007$ & $3.523 \pm 0.635$ & $0.409 \pm 0.141$ \\ 
                        CMUsubject16 & $0.088 \pm 0.020$ & $0.040 \pm 0.009$ & $0.089 \pm 0.027$ & $0.270 \pm 0.080$ & $0.089 \pm 0.039$ & $0.255 \pm 0.002$ \\ 
                        DigitShapes & $0.008 \pm 0.001$ & $0.035 \pm 0.051$ & $0.021 \pm 0.001$ & $0.013 \pm 0.002$ & $0.727 \pm 0.569$ & $0.035 \pm 0.003$ \\ 
                        ECG & $0.402 \pm 0.023$ & $0.431 \pm 0.037$ & $0.356 \pm 0.008$ & $0.496 \pm 0.018$ & $0.601 \pm 0.137$ & $0.543 \pm 0.019$ \\ 
                        Jap.~Vowels & $0.080 \pm 0.031$ & $0.053 \pm 0.009$ & $0.069 \pm 0.003$ & $0.061 \pm 0.029$ & $0.052 \pm 0.005$ & $0.067 \pm 0.001$ \\ 
                        Kick vs Punch & $0.301 \pm 0.109$ & $0.493 \pm 0.128$ & $0.224 \pm 0.014$ & $0.696 \pm 0.046$ & $0.674 \pm 0.037$ & $0.662 \pm 0.017$ \\ 
                        LIBRAS & $0.320 \pm 0.045$ & $0.346 \pm 0.091$ & $0.259 \pm 0.021$ & $0.911 \pm 0.056$ & $1.110 \pm 0.248$ & $1.608 \pm 0.311$ \\ 
                        NetFlow & $0.218 \pm 0.009$ & $0.259 \pm 0.078$ & $0.189 \pm 0.014$ & $0.251 \pm 0.041$ & $0.194 \pm 0.011$ & $0.168 \pm 0.081$ \\ 
                        PEMS & $0.704 \pm 0.130$ & $1.100 \pm 0.064$ & $0.520 \pm 0.058$ & $1.194 \pm 0.308$ & $0.784 \pm 0.111$ & $0.537 \pm 0.010$ \\ 
                        PenDigits & $0.289 \pm 0.127$ & $0.399 \pm 0.206$ & $0.146 \pm 0.007$ & $0.185 \pm 0.027$ & $0.187 \pm 0.043$ & $0.181 \pm 0.005$ \\ 
                        Shapes & $0.014 \pm 0.004$ & $0.012 \pm 0.004$ & $0.011 \pm 0.002$ & $0.016 \pm 0.008$ & $0.168 \pm 0.142$ & $0.012 \pm 0.001$ \\ 
                        UWave & $0.113 \pm 0.011$ & $0.121 \pm 0.017$ & $0.140 \pm 0.004$ & $0.745 \pm 0.151$ & $1.168 \pm 1.063$ & $0.189 \pm 0.008$ \\ 
                        Wafer & $0.048 \pm 0.021$ & $0.081 \pm 0.011$ & $0.105 \pm 0.010$ & $0.105 \pm 0.086$ & $0.029 \pm 0.011$ & $0.085 \pm 0.002$ \\ 
                        Walk vs Run & $0.030 \pm 0.008$ & $0.030 \pm 0.008$ & $0.023 \pm 0.007$ & $0.048 \pm 0.040$ & $0.028 \pm 0.000$ & $0.066 \pm 0.001$ \\
                        \midrule
                        Mean nlpp. & $0.175$ & $0.238$ & $0.180$ & $0.514$ & $0.603$ & $0.423$ \\
                        Med.~nlpp. & $0.097$ & $0.122$ & $0.124$ & $0.261$ & $0.221$ & $0.185$ \\
                        Sd.~nlpp. & $0.183$ & $0.273$ & $0.161$ & $0.623$ & $0.841$ & $0.542$ \\
                        \midrule
                        Mean rank ($n_\bX < 300$) & $2.800$ & $2.900$ & $2.200$ & $4.700$ & $4.000$ & $4.400$ \\ 
                        Mean rank ($n_\bX \geq 300$) & $2.333$ & $3.333$ & $2.833$ & $4.833$ & $4.333$ & $3.333$ \\ 
                        Mean rank (all) & $2.625$ & $3.062$ & $2.438$ & $4.750$ & $4.125$ & $4.000$ \\
                    \bottomrule
				\end{tabular}
			\end{sc}
		\end{small}
	\end{center}
	% \vskip -1.0in
\end{table*}

\begin{table*}[h]
	\caption{Mean and standard deviation of accuracies on test sets over $5$ independent runs}
	\label{table:full_acc_results}
	\vskip 0.15in
	\begin{center}
		\begin{small}
			\begin{sc}
				\begin{tabular}{lcccccc}%{lllll}
					\toprule
					Dataset & GP-Sig-LSTM & GP-Sig-GRU & GP-Sig & GP-LSTM  & GP-GRU & GP-KConv1D\\
					\midrule 
                        Arabic Digits & $0.992 \pm 0.003$ & $0.994 \pm 0.002$ & $0.979 \pm 0.004$ & $0.985 \pm 0.004$ & $0.986 \pm 0.005$ & $0.984 \pm 0.001$ \\ 
                        AUSLAN & $0.983 \pm 0.003$ & $0.978 \pm 0.006$ & $0.925 \pm 0.014$ & $0.880 \pm 0.012$ & $0.949 \pm 0.014$ & $0.784 \pm 0.012$ \\ 
                        Character Traj. & $0.991 \pm 0.003$ & $0.925 \pm 0.078$ & $0.979 \pm 0.002$ & $0.233 \pm 0.331$ & $0.114 \pm 0.050$ & $0.941 \pm 0.013$ \\ 
                        CMUsubject16 & $1.000 \pm 0.000$ & $1.000 \pm 0.000$ & $0.979 \pm 0.017$ & $0.924 \pm 0.051$ & $0.993 \pm 0.014$ & $0.897 \pm 0.000$ \\ 
                        DigitShapes & $1.000 \pm 0.000$ & $0.988 \pm 0.025$ & $1.000 \pm 0.000$ & $1.000 \pm 0.000$ & $0.812 \pm 0.153$ & $1.000 \pm 0.000$ \\ 
                        ECG & $0.816 \pm 0.029$ & $0.832 \pm 0.012$ & $0.848 \pm 0.010$ & $0.782 \pm 0.032$ & $0.734 \pm 0.033$ & $0.760 \pm 0.018$ \\ 
                        Jap.~Vowels & $0.981 \pm 0.005$ & $0.985 \pm 0.004$ & $0.982 \pm 0.005$ & $0.982 \pm 0.004$ & $0.986 \pm 0.005$ & $0.986 \pm 0.002$ \\ 
                        Kick vs Punch & $0.900 \pm 0.063$ & $0.820 \pm 0.098$ & $0.900 \pm 0.000$ & $0.620 \pm 0.075$ & $0.600 \pm 0.110$ & $0.700 \pm 0.089$ \\ 
                        LIBRAS & $0.921 \pm 0.013$ & $0.899 \pm 0.031$ & $0.923 \pm 0.004$ & $0.776 \pm 0.019$ & $0.742 \pm 0.050$ & $0.698 \pm 0.026$ \\ 
                        NetFlow & $0.931 \pm 0.002$ & $0.921 \pm 0.012$ & $0.937 \pm 0.003$ & $0.928 \pm 0.011$ & $0.926 \pm 0.012$ & $0.945 \pm 0.027$ \\ 
                        PEMS & $0.763 \pm 0.016$ & $0.775 \pm 0.019$ & $0.820 \pm 0.014$ & $0.745 \pm 0.044$ & $0.769 \pm 0.020$ & $0.794 \pm 0.008$ \\ 
                        PenDigits & $0.928 \pm 0.030$ & $0.902 \pm 0.048$ & $0.955 \pm 0.002$ & $0.953 \pm 0.008$ & $0.951 \pm 0.008$ & $0.946 \pm 0.001$ \\ 
                        Shapes & $1.000 \pm 0.000$ & $1.000 \pm 0.000$ & $1.000 \pm 0.000$ & $1.000 \pm 0.000$ & $0.867 \pm 0.163$ & $1.000 \pm 0.000$ \\ 
                        UWave & $0.970 \pm 0.004$ & $0.968 \pm 0.006$ & $0.964 \pm 0.001$ & $0.870 \pm 0.029$ & $0.763 \pm 0.225$ & $0.947 \pm 0.002$ \\ 
                        Wafer & $0.988 \pm 0.005$ & $0.978 \pm 0.005$ & $0.965 \pm 0.004$ & $0.966 \pm 0.037$ & $0.994 \pm 0.002$ & $0.984 \pm 0.001$ \\ 
                        Walk vs Run & $1.000 \pm 0.000$ & $1.000 \pm 0.000$ & $1.000 \pm 0.000$ & $1.000 \pm 0.000$ & $1.000 \pm 0.000$ & $1.000 \pm 0.000$ \\
                        \midrule
                        Mean acc. & $0.948$ & $0.935$ & $0.947$ & $0.853$ & $0.824$ & $0.898$ \\
                        Med.~acc. & $0.982$ & $0.973$ & $0.964$ & $0.926$ & $0.896$ & $0.946$ \\
                        Sd.~acc. & $0.068$ & $0.070$ & $0.052$ & $0.193$ & $0.218$ & $0.107$ \\
                        \midrule
                        Mean rank ($n_\bX < 300$) & $3.000$ & $3.100$ & $2.800$ & $4.250$ & $4.250$ & $3.600$ \\ 
                        Mean rank ($n_\bX \geq 300$) & $2.167$ & $3.500$ & $3.000$ & $4.167$ & $4.333$ & $3.833$ \\ 
                        Mean rank (all) & $2.688$ & $3.250$ & $2.875$ & $4.219$ & $4.281$ & $3.688$ \\
					\bottomrule
				\end{tabular}
			\end{sc}
		\end{small}
	\end{center}
	% \vskip -1.0in
\end{table*}

\begin{table*}[h]
	\caption{Accuracies of frequentist time series classification methods}
	\label{table:freq_acc_results}
	\vskip 0.15in
	\begin{center}
		\begin{small}
			\begin{sc}
				\begin{tabular}{lrrrrrrrr}%{lllll}
					\toprule
					Dataset & SMTS & LPS & mvARF & DTW & ARKernel & gRSF & MLSTMFCN & MUSE \\
					\midrule
					Arabic Digits & $0.964$ & $0.971$ & $0.952$ & $0.908$ & $0.988$ & $0.975$ & $0.990$ & $0.992$ \\
					AUSLAN & $0.947$ & $0.754$ & $0.934$ & $ 0.727$ & $0.918$ & $0.955$ & $0.950$ & $0.970$ \\
					Character Traj. & $0.992$ & $0.965$ & $0.928$ & $0.948$ & $0.900$ & $0.994$ & $0.990$ & $0.937$ \\
					CMUsubject16 & $0.997$ & $1.000$ & $1.000$ & $0.930$ & $1.000$ & $1.000$ & $1.000$ & $1.000$ \\
					DigitShapes & $1.000$ & $1.000$ & $1.000$ & $1.000$ & $1.000$ & $1.000$ & $1.000$ & $1.000$ \\
					ECG & $0.818$ & $0.820$ & $0.785$ & $0.790$ & $0.820$ & $0.880$ & $0.870$ & $0.880$ \\
					Jap.~Vowels & $0.969$ & $0.951$ & $0.959$ & $0.962$ & $0.984$ & $0.800$ & $1.000$ & $0.976$ \\
					Kick vs Punch & $0.820$ & $0.900$ & $0.976$ & $0.600$ & $0.927$ & $1.000$ & $0.900$ & $1.000$ \\
					LIBRAS & $0.909$ & $0.903$ & $0.945$ & $0.888$ & $0.952$ & $0.911$ & $0.970$ & $0.894$ \\
					NetFlow & $0.977$ & $0.968$ & NA & $0.976$ & NA & $0.914$ & $0.950$ & $0.961$ \\
					PEMS & $0.896$ & $0.844$ & NA & $0.832$ & $0.750$ & $1.000$ & NA & NA \\
					PenDigits & $0.917$ & $0.908$ & $0.923$ & $0.927$ & $0.952$ & $0.932$ & $0.970$ & $0.912$ \\
					Shapes & $1.000$ & $1.000$ & $1.000$ & $1.000$ & $1.000$ & $1.000$ & $1.000$ & $1.000$   \\
					UWave & $0.941$ & $0.980$ & $0.952$ & $0.916$ & $0.904$ & $0.929$ & $0.970$ & $0.916$ \\
					Wafer & $0.965$ & $0.962$ & $0.931$ & $0.974$ & $0.968$ & $0.992$ & $0.990$ & $0.997$ \\
					Walk vs Run & $1.000$ & $1.000$ & $1.000$ & $1.000$ & $1.000$ & $1.000$ & $1.000$ & $1.000$ \\
					\midrule
					Mean acc. & $0.945$ & $0.933$ & $0.949$ & $0.899$ & $0.938$ & $0.955$ & $0.970$ & $0.962$ \\
                    Med.~acc. & $0.964$ & $0.964$ & $0.952$ & $0.929$ & $0.952$ & $0.984$ & $0.990$ & $0.976$ \\
                    Sd.~acc. & $0.059$ & $0.073$ & $0.055$ & $0.111$ & $0.073$ & $0.058$ & $0.039$ & $0.043$ \\
					\bottomrule
				\end{tabular}
			\end{sc}
		\end{small}
	\end{center}
	% \vskip -1.0in
\end{table*}

\clearpage

\begin{figure*}[t]
    \centering
    \begin{minipage}{0.29\textwidth}
        \centering
        \includegraphics[height=1.85in]{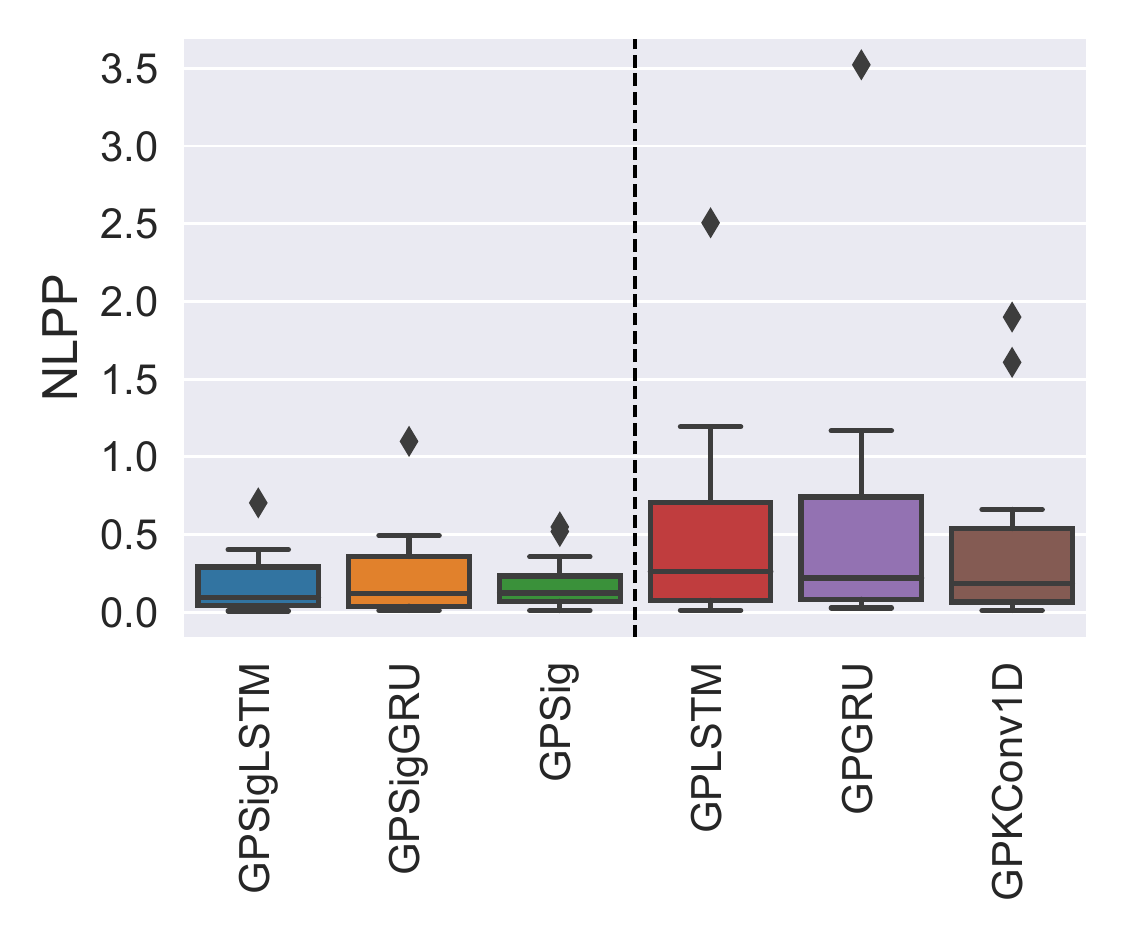}
    \end{minipage}
    \begin{minipage}{0.69\textwidth}
        \centering
        \includegraphics[height=1.85in]{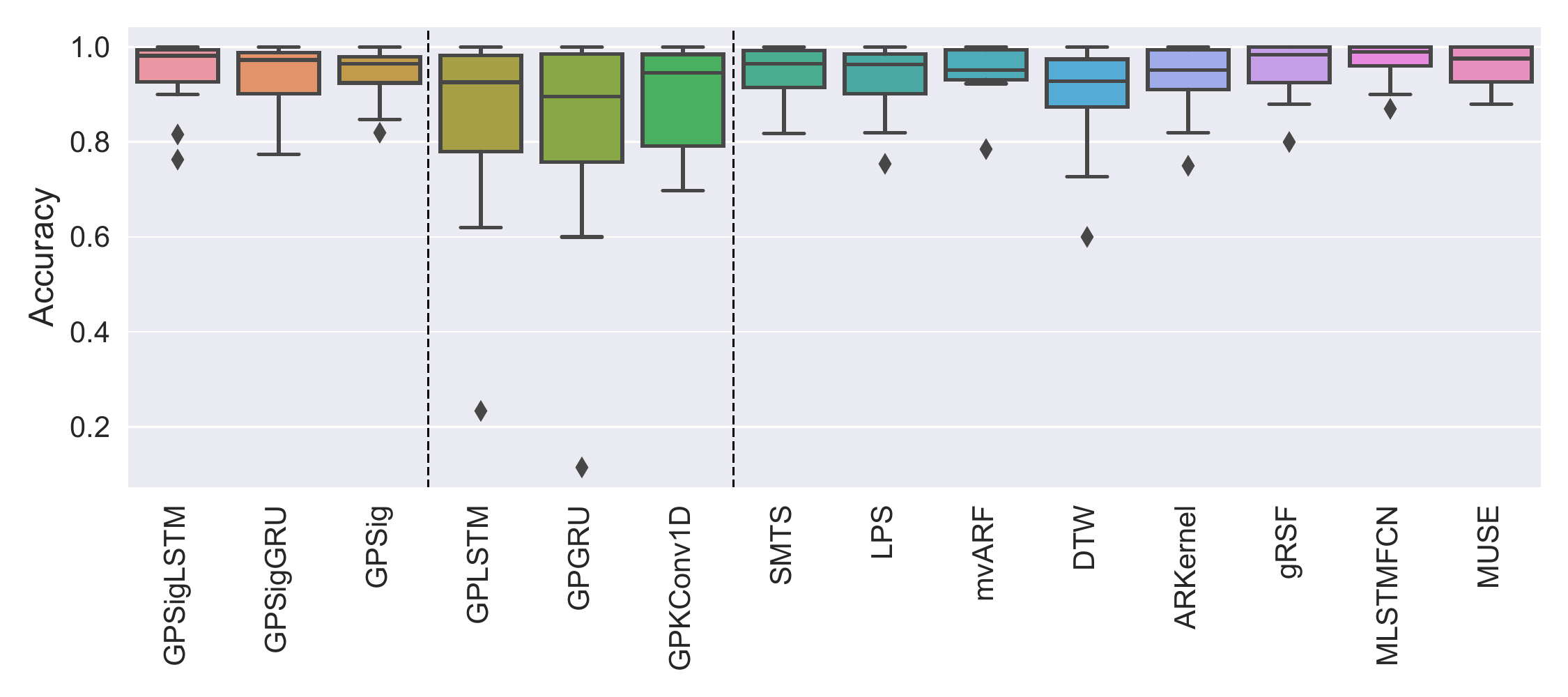}
    \end{minipage}
    \caption{Box-plots of negative log-predictive probabilities (left) and classification accuracies (right) on 16 TSC datasets}
    \label{fig:boxplots}
\end{figure*}

\clearpage

% \putbib
% \end{bibunit}
 
\end{document}